\documentclass[11pt]{article}
\usepackage[margin=1in]{geometry}

\usepackage[utf8]{inputenc}
\usepackage[T1]{fontenc}
\usepackage[sc]{mathpazo}

\RequirePackage{algorithm}
\RequirePackage{algorithmic}
\RequirePackage{natbib}

\usepackage{lipsum}  
\usepackage{titling} 

\usepackage{graphicx, amsmath, amsthm, amssymb, enumitem, mathrsfs} 
\usepackage{booktabs} 
\usepackage{lipsum}
\usepackage{hyperref}

\usepackage{amsmath}
\usepackage{amssymb}
\usepackage{mathtools}
\usepackage{amsthm}
\usepackage{arydshln}
\usepackage{caption}
\usepackage{subcaption}

\usepackage[capitalize,noabbrev]{cleveref}

\theoremstyle{plain}
\newtheorem{theorem}{Theorem}[section]
\newtheorem{proposition}[theorem]{Proposition}
\newtheorem{lemma}[theorem]{Lemma}
\newtheorem{corollary}[theorem]{Corollary}
\theoremstyle{definition}
\newtheorem{definition}[theorem]{Definition}

\theoremstyle{remark}
\newtheorem{remark}[theorem]{Remark}

\DeclareMathOperator*{\argmax}{arg\,max}
\DeclareMathOperator*{\argmin}{arg\,min}
\newcommand{\cp}{{\mathcal{C}_p}}

\newcommand{\supp}{{\operatorname{supp}}}

\newcommand{\pr}[1]{\mathbb{P}\left({#1}\right)}
\newcommand{\ep}[1]{\mathbb{E}\left({#1}\right)}

\usepackage{xcolor, soul}

\title{Robust Graph-Based Semi-Supervised Learning\\via $p$-Conductances}
\author{Sawyer Jack Robertson\thanks{Equal contribution} \thanks{\footnotesize Department of Mathematics, UC San Diego} \thanks{Halicio{\u g}lu Data Science Institute, UC San Diego} , Chester Holtz${}^\ast {}^\ddagger$, Zhengchao Wan\thanks{Department of Mathematics, University of Missouri} , \\
Gal Mishne${}^\ddagger$, Alexander Cloninger${}^\dagger{}^\ddagger$}
\date{}

\begin{document}

\maketitle

\vspace*{-.5in}
\begin{abstract}
    We study the problem of semi-supervised learning on graphs in the regime where data labels are scarce or possibly corrupted. We propose an approach called $p$-conductance learning that generalizes the $p$-Laplace and Poisson learning methods by introducing an objective reminiscent of $p$-Laplacian regularization and an affine relaxation of the label constraints. This leads to a family of probability measure mincut programs that balance sparse edge removal with accurate distribution separation. Our theoretical analysis connects these programs to well-known variational and probabilistic problems on graphs (including randomized cuts, effective resistance, and Wasserstein distance) and provides motivation for robustness when labels are diffused via the heat kernel. Computationally, we develop a semismooth Newton–conjugate gradient algorithm and extend it to incorporate class-size estimates when converting the continuous solutions into label assignments. Empirical results on computer vision and citation datasets demonstrate that our approach achieves state-of-the-art accuracy in low label-rate, corrupted-label, and partial-label regimes.

    {\footnotesize \textbf{\texttt{Keywords:}} semi-supervised learning, node classification, robust machine learning, graph effective resistance, optimal transport}

    {\footnotesize \textbf{\texttt{MSC2020:}} 68Q87, 90C35, 05C90, 90C53}
\end{abstract}

\section{Introduction}
    Semi-supervised learning (SSL) algorithms exploit both labeled and unlabeled examples to learn predictors that achieve higher classification accuracy compared to those learned from the labeled data alone. This paradigm is motivated by the high cost of obtaining annotated data in real-world contexts. In the ``low label-rate regime,'' where very few labeled data are available, relationships among and affinities between the unlabeled data can be used to significantly improve the prediction accuracy of classic machine learning models. For example, the seminal work of \cite{zhu2003semi} proposed Laplace learning, which obtains class predictions from harmonic (or minimial energy) extensions of the available labels. Given a weighted graph $G=(V, E, w)$ on $n$ nodes and an initial set of $m\ll n$ binary node labels $y_i \in \{-1, 1\}^{m}$,~\cite{zhu2003semi} extend $y$ via the minimizer of the following program with constraints $\phi_i = y_i$:
    \begin{equation}\label{eq:laplace}
                \min_{\phi \in \mathbb{R}^{n}} \left\{ \phi^\top L\phi : \phi_i = y_i,\: 1\leq i \leq m \right\}.
        \end{equation}
    where $L$ is a graph Laplacian matrix. Despite the effectiveness of the method when a sufficient proportion of the data is labeled initially, it is well-known that as the ratio between the number of labeled vertices and unlabeled vertices approaches zero, methods based on Laplace learning yield solutions that degenerate towards a constant, as first characterized in~\cite{nadler2009limit}. More recently, various techniques have been proposed to resolve this degeneracy~\cite{calder2022hamilton,calder20poisson, flores2018algorithms, zhu2005semi, zhou2011ssl, holtz2024continuous}. For example, in Poisson Learning,~\cite{calder20poisson} precisely characterize the degeneracy of Laplace learning from the perspective of random walks over the graph. This interpretation motivated the authors to solve the following regularized problem with adjusted labels:
        \begin{equation}
            \min_{\phi \in \mathbb{R}^n}\phi^\top L \phi - \sum_{i=1}^m(y_i - \bar{y})\cdot \phi_i.
        \end{equation}
    Recent work has also considered $p$-Laplace learning \cite{alaoui2016asymptotic, flores2022}, which solves the following variational problem
        \begin{equation}\label{eq:p-laplace-learning}
            \min_{\phi\in \mathbb{R}^n} \left \{\sum_{i=1}^n w_{ij}|\phi_i - \phi_j|^p : \phi_i = y_i, \:\: 1 \leq i \leq m \right \}.
        \end{equation}
    and which recovers the Laplace learning method when $p=2$. While several works have explored the problem for specific $p$, e.g. $p=1$ \cite{jungsparselp2016} and $p=\infty$ (i.e. Lipschitz learning) \cite{kyng2015lipschitz,calder2019lipschitz}, more recent work has focused on $p$-Laplace regularization for general $p$, particularly in the low label-rate regime \cite{flores2022}, and recent numerical results show that $p > 2$ is superior to Laplace learning for graph-based SSL in this setting \cite{flores2022}. In related work, ~\cite{calder2022hamilton} analyzed applications of the $p$-eikonal equations $\sum_{j=1}^nw_{ij}(\phi_i - \phi_j)_+^p = g(\phi_i)$, where $g : V \to \mathbb{R}$. 

    From a computational perspective, related works include computation of the $p$-Laplacian and $p$-effective resistance~\cite{saito2023multi}, both of which involve minimizing a weighted $p$-norm. In the cases, when $p \in [2,\infty)$, the weighted $p$-norm is smooth and convex, but not strongly convex for $p>2$. First-order proximal methods are a natural choice, exhibiting linear convergence as shown in \cite{flores2018algorithms}. In~\cite{flores2022}, Newton methods with homotopy on $p$ were suggested for computing solutions to the $p$-Laplace equations~\eqref{eq:p-laplace-learning}. Recently, iteratively reweighted least squares were studied in depth in~\cite{adil2019irls}. Notably, the aforementioned methods struggle when $p\in \{1,\infty\}$. In contrast, we describe a globally superlinearly convergent method that works for any $p\in [1,\infty]$.

    While existing methods perform well across label rate regimes, we are not aware of approaches which are explicitly designed with robustness to label corruptions in mind. We propose a method that addresses this gap and which generalizes both Poisson learning and $p$-Laplace learning.

\subsection{Contributions}
   
    In this paper, we consider a variant of Poisson and $p$-Laplace learning by considering a natural affine relaxation of the equality constraint $\phi_i=y_i$. In the binary classification setting, we model the initial labels in the form of two probability measures $\mu, \nu\in\mathbb{R}^n$, respectively. We seek to solve the following program:
        \begin{align}\label{eq:bcut-intro}\tag{$\mathcal{C}_p$}
            \mathcal{C}_p(\mu,\nu) &= \min_{\substack{\phi \in \mathbb{R}^n \\ \phi^T(\mu-\nu) = 1}}(\sum_{\{i, j\} \in E}w_{ij}|\phi_i - \phi_j|^p)^{1/p},
        \end{align}
    for $p\in[1,\infty)$, and for $p=\infty$, we consider the program
        \begin{align}\label{eq:bcut-intro-infty}\tag{$\mathcal{C}_\infty$}
            \mathcal{C}_\infty(\mu,\nu) &= \min_{\substack{\phi \in \mathbb{R}^n \\ \phi^T(\mu-\nu) = 1}} \max_{\{i, j\} \in E}w_{ij}|\phi_i - \phi_j|.
        \end{align}
    One may think of these programs as searching for a vector $\phi$ which satisfies $\phi \approx \frac{1}{2}$ near $\supp(\mu)$ and $\phi \approx -\frac{1}{2}$ near $\supp(\nu)$, and such that $\phi$ has minimal energy with respect to the $p$-Laplacian operator on the graph. We can then produce class predictions by taking $\widehat{y} = \mathrm{sgn}(\phi^\ast)$ where $\phi^\ast$ is any mean-centered solution to $\cp$. In the setting of greater than two classes, we consider a ``one-vs-all'' formulation of the problem to produce a prediction among all of the classes (see \cref{rmk:one-vs-all} and \cref{fig:illustration-cuts-digit}). 

    We term $\cp(\mu,\nu)$ the measure $p$-conductance between $\mu,\nu$, reflecting its reciprocal relationship with measure effective resistance $(\mu-\nu)^T L^+ (\mu-\nu)$ when $p=2$ (see \cref{thm:generalized-gauge-duality} and \cref{cor:er}), and we refer to our method of obtaining class predictions from potentials $\phi$ as $p$-conductance learning (see \cref{rmk:class-pred}). The parameter $p$ can provide control on the sparsity of the solutions $\phi^\ast$, which tend to be concentrated on fewer nodes when $p = 1$ and many nodes as $p\rightarrow\infty$. In \cref{fig:illustration-cuts-demo}, we illustrate solutions to the program $\cp$ for choices of $\mu,\nu$ and $p$ on a toy graph.

    \begin{figure*}[t!]
        \begin{center}
            \includegraphics[width=0.95\linewidth]{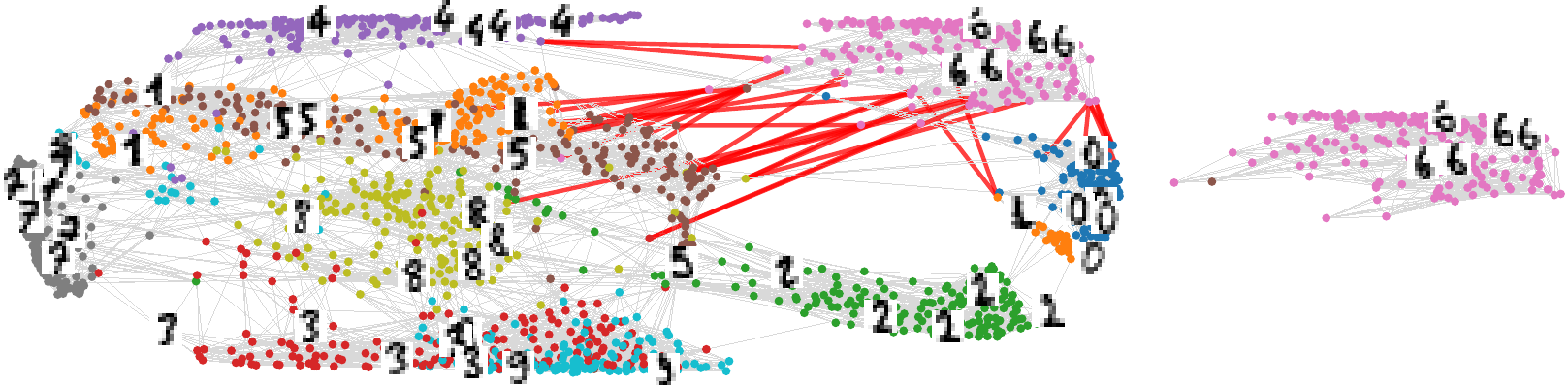}
        \end{center}\caption{
        \textbf{(Left)} One-vs-all measure \texttt{mincut} on the Sklearn Digits dataset~\cite{digits}. The initially labeled nodes are shown with images overlaid. $\mu$ is given by the five images of the digit six, and $\nu$ from all of the other classes. Solving the program~\eqref{eq:bcut-intro} for $p=1$, we obtain a sparse solution $\phi$ such that each $\{i, j\}\in E$ satisfies $|\phi_i-\phi_j| \approx 0$ (light gray) or $|\phi_i-\phi_j| \approx 0.11$ (red). \textbf{(Right)} Removing red edges isolates a connected component with near-perfect class separation.}
        \label{fig:illustration-cuts-digit}\vspace{-.5cm}
    \end{figure*}

    These programs demonstrate crisp theoretical properties which allow for a high degree of interpretability. Namely, we model this class of problems as a family of so-called ``measure minimum cuts,'' wherein the vector $\phi$, which we term a potential, is used to define an edge cut on the underlying graph and with respect to which the objective $\cp$ can be understood to be targeting competing localization and cut sparsity constraints. Our contributions are as follows. 
    
    \begin{itemize}
        \item In \cref{subsec:relations}, we show that for $p=1$ our programs yield a generalized version of the classical min-cut max-flow duality (\cref{th:grl-max-flow-min-cut}) and for $p=1, 2$, respectively, the program may be reformulated in terms of random cuts and normalized cuts of the graph $G$ (\cref{thm:mincuts-random-cuts} and \cref{rmk:normalized-measure-cut}). We then obtain direct relationships between $\mathcal{C}_p$, effective resistance, and Wasserstein distance in the settings of $p=2,\infty$, respectively (\cref{cor:w1} and \cref{cor:er}). 
        \item In \cref{subsec:robustness} we establish theoretical motivation for the robustness of class predictions produced by our method particularly when the initial labels are subject to a small diffusion step (\cref{thm:robustness}).
        \item In \cref{sec:algorithms}, we obtain fast algorithms to solve $\mathcal{C}_p$ which search for saddle points of the the augmented Lagrangian; namely, we introduce the semismooth Newton Conjugate Gradient method to solve the nonsmooth Newton equations (\cref{alg:alg1} and \cref{alg:alg2}). Then, we establish its convergence (\cref{thm:ssnal-convergence}), and present algorithms for cut-assignment when estimates of class cardinalities are provided. 
        \item In \cref{sec:experimental-results}, we present experiments on citation networks (Cora, Citeseer, Pubmed) and image datasets (MNIST, FashionMNIST, CIFAR-10, and CIFAR-100), on which our approach achieves state-of-the-art results.
    \end{itemize}

\section{Theoretical analysis of $p$-Conductances}\label{sec:properties-and-theory}

    In this section, we study the theoretical properties of the $p$-conductance program, focusing on its links to classical graph optimization problems and on stability of optimal potentials for $p=2$, in particular when the diffusion operator $e^{-tL}$ is applied as a pre-processing step.

    As before, $G=(V, E, w)$ is a weighted undirected graph on $n$ nodes and $N$ edges with symmetric edge weights $w_{ij}$. We denote by $E'$ the set of oriented edges, i.e., $E' = \{(i, j), (j, i): \{i, j\}\in E\}$. We define the degree of each node $i\in V$ by $d_i = \sum_{j=1}^n w_{ij}$, and the Laplacian matrix $L=D-W$ where $W$ is the weighted adjacency matrix of $G$. If $A, B\subseteq V$ we denote by $\#E(A, B)$ the sum of the weights of edges with one endpoint in each of $A,B$.
    
    We denote by $\mathsf{P}(V)$ the probability simplex of $V$, and model the initial labels in the form of two measures $\mu,\nu\in\mathsf{P}(V)$ which correspond either to sum-normalized indicators of the labelled nodes in each of two classes, or the measures obtained from the one-vs-all setup as described in \cref{rmk:one-vs-all}. Proofs of all results can be found in \cref{sec:proofs-sec-2}.

\subsection{Relating $p$-conductances to other problems}\label{subsec:relations}
    
    Recall that the node \texttt{mincut} problem seeks a bipartition $V_s, V_t \subseteq V$ for fixed nodes $s, t \in V$, such that $s \in V_s$, $t \in V_t$, and with minimal crossings $\# E(V_s, V_t)$. For disjointly supported $\mu, \nu \in \mathsf{P}(V)$, the problem extends naturally but can degenerate under strict interpretations (e.g., when unique edge crossings occur, see \cref{fig:illustration-cuts-demo}). A relaxation is achieved by $\mathcal{C}_1(\mu,\nu)$. 

    \begin{theorem}[Generalized Max Flow - Min Cut Theorem]\label{th:grl-max-flow-min-cut}
        Assume $G$ is connected and let $\mu,\nu\in\mathsf{P}(V)$. The program $\mathcal{C}_1(\mu,\nu)$ is realizable as $\mathtt{mincut}(\mu,\nu)$, which is a linear program in the variables $k=(k_{ij})\in \mathbb{R}^{2N}$ and $\psi\in\mathbb{R}^n$:
            \begin{align}\label{eq:linear-program-c1}
                \mathtt{mincut}(\mu,\nu) = \begin{cases}
                    \text{minimize}&\sum_{(i, j)\in E'}w_{ij}k_{ij}\\
                    \text{subject to }& k_{ij}\geq 0\\
                    &k_{ij} \geq \psi_i - \psi_j\\
                    & (\mu-\nu)^T \psi \geq 1
                \end{cases}.
            \end{align}
        Moreover, $\mathtt{mincut}(\mu,\nu)$ admits a dual formulation in the variables $J=(J_{ij})\in \mathbb{R}^{E'}$ and $f\in\mathbb{R}$:
            \begin{align}\label{eq:dual-linear-program-c1}
                \mathtt{maxflow}(\mu,\nu) = \begin{cases}
                    \text{maximize }&f\\
                    \text{subject to }&f\geq 0\\
                    &J_{ij}\geq 0\\
                    &J_{ij}\leq w_{ij}\\
                    &\widetilde{B}J = f(\mu-\nu)
                \end{cases}.
            \end{align}
        Here, $\widetilde{B}\in\mathbb{R}^{n\times 2N}$ is the oriented incidence matrix according to $E'$. Finally, strong duality holds.
    \end{theorem}

    When only one node is labeled in each class, $\mathtt{mincut}(\mu,\nu)$ recovers the usual node \texttt{mincut} problem. Another angle from which to view $\mathcal{C}_1$ is that of randomized cuts.

    \begin{theorem}[$\mathcal{C}_1$ via randomized cuts]\label{thm:mincuts-random-cuts}
        Let $\mu,\nu\in\mathsf{P}(V)$ and $\phi\in\mathbb{R}^n$. Let $T\sim \mathrm{Unif}([0,\|\phi\|_\infty])$ and set $A_T = \{i\in V: \phi_i\geq T\}$ and $B_T = \{i\in V: \phi_i< T\}$. Then $\mathtt{mincut}(\mu,\nu)$ may be recast in terms of randomized cuts:
            \begin{align*}
                \mathtt{mincut}(\mu,\nu)=
                \begin{cases}
                    \text{minimize }& \frac{\mathbb{E}_T({\#E(A_T, B_T)})}{\mathbb{E}_T({\sum_{i\in A_T}(\mu_i-\nu_i)})} \\
                    \text{subject to }&\phi \geq 0\\
                    &\phi^T(\mu-\nu) = 1
                \end{cases}.
            \end{align*}
    \end{theorem}

    We illustrate the measure \texttt{mincut} problem on a toy dataset in~\cref{fig:illustration-cuts-digit}. When $p=2$, we may recast $\mathcal{C}_2$ as a relaxation of a variant of the normalized cut problem defined via the labels $\mu,\nu$, stated as a remark below.
    
    \begin{remark}[$\mathcal{C}_2$ via normalized cuts]\label{rmk:normalized-measure-cut}
        Let $\mu,\nu\in\mathsf{P}(V)$. We consider the following measure normalized cut problem:
            \begin{align}\label{eq:measure-ncut}
                \mathtt{ncut}({\mu,\nu})&=  \begin{cases}
                    \text{minimize }& \frac{1}{2}\bigg{\{} \frac{|\#E(A,A^c)| }{\sum_{i\in A}(\mu_i-\nu_i)}\\
                    &\qquad+ \frac{|\#E(A,A^c)| }{\sum_{i\in A^c}(\nu_i-\mu_i)} \bigg{\}} \\
                    \text{subject to }& \varnothing \subsetneq A \subsetneq V\\
                    &\mathbf{1}_A^T(\mu-\nu) >0
                \end{cases}.
            \end{align}
        Then $\mathcal{C}_2(\mu,\nu)$ is a quadratic program with objective $\phi^T L\phi$ and can be realized as a linear relaxation of \cref{eq:measure-ncut} by considering the indicator vector of the set $A$ and relaxing this to $\phi^T(\mu-\nu) = 1$.
    \end{remark}   

    \begin{figure}[t!]
        \begin{center}
            \includegraphics[width=0.75\textwidth]{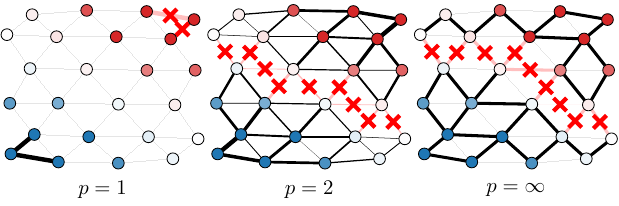}
        \end{center}\caption{Visualizing solutions $\phi$ to $\cp$ for $p \in [1,\infty]$. Measures $\mu, \nu$ appear in blue and red, with opacity proportional to value. Edge width is proportional to $|\phi_i - \phi_j|$. ``X'' symbols mark edges crossing $\{\phi \le T\}$ and $\{\phi > T\}$, where $T$ is the mean of $\phi$.}
        \label{fig:illustration-cuts-demo}
    \vspace{-0.4cm}
    \end{figure}

    The $p$-conductance problem is also related to a family of optimal transportation metrics on $\mathsf{P}(V)$. Define the $p$-Beckmann metric (see \cite{beckmann1952continuous}), denoted $\mathcal{B}_{w, p}(\mu,\nu)$, as the optimal value of the following program, where $1\leq p<\infty$:
        \begin{align}\label{eq:defn-beckmann-1}\tag{$\mathcal{B}_{w, p}$}
            \mathcal{B}_{w, p}(\mu,\nu) &= \min_{\substack{J\in\mathbb{R}^{N} \\ BJ = \mu-\nu}}(\sum_{\{i, j\} \in E}w_{ij}|J(i, j)|^p)^{1/p}.
        \end{align}
    When $p=\infty$, we set
        \begin{align}\label{eq:defn-beckmann-2}\tag{$\mathcal{B}_{w, \infty}$}
            \mathcal{B}_{w, \infty}(\mu,\nu) &= \min_{\substack{J\in\mathbb{R}^{N} \\ BJ = \mu-\nu}}\max_{\{i, j\} \in E}w_{ij}|J(i, j)|.
        \end{align}
    Here, $B$ is the node-edge oriented incidence matrix of $G$. This class of problems was considered in, for example,~\cite{robertson2024all} as a family of distances that are related to the effective resistance metric of the underlying graph. From this setup we have the following theorem relating $\mathcal{B}_{w, p}(\mu,\nu)$ to $\cp$.

    \begin{theorem}\label{thm:generalized-gauge-duality}
        Let $\mu,\nu\in\mathsf{P}(V)$ (with $\mu\neq\nu$). Then the optimal values of $\cp$ and $\mathcal{B}_{w, p}$ are related as follows:
            \begin{align*}
                \mathcal{C}_p&(\mu,\nu) = 
                \begin{cases}
                    1 / \mathcal{B}_{\infty, w^{-1}}(\mu,\nu) &\text{ if }p=1,\\
                    1 / \mathcal{B}_{q, w^{1-q}}(\mu,\nu) &\text{ if }p\in(1, \infty)\\
                    &\text{ and }1/p+1/q=1,\\
                    1 / \mathcal{B}_{1, w^{-1}}(\mu,\nu) &\text{ if }p=\infty.
                \end{cases}
            \end{align*}
    \end{theorem}

    The distinctive appearance of reciprocality between the programs $\cp$ and $\mathcal{B}_{w, p}$ follows from the strong gauge duality for linearly constrained norm optimization problems (see~\cite{friedlander2014gauge}, a more detailed discussion is given in \cref{subsec:gauge}). By recalling the well-known fact that $\mathcal{B}_{1, w}(\mu, \nu)$ is the 1-Wasserstein metric between $\mu,\nu$ when the ground metric on $V$ is weighted shortest path distance (see, e.g., \cite{peyre2019computational}), and which we denote $\mathcal{W}_{1, w}(\mu, \nu)$, we may connect $\mathcal{C}_\infty$ to a certain 1-Wasserstein metric, as follows.

    \begin{corollary}\label{cor:w1}
        Let $\mu,\nu\in\mathsf{P}(V)$ with $\mu\neq\nu$. Then it holds
            \begin{align*}
                \mathcal{C}_{\infty}(\mu,\nu) &= 1 / \mathcal{W}_{1, w^{-1}}(\mu, \nu).
            \end{align*}
        That, is $\mathcal{C}_{\infty}(\mu,\nu)$ is the reciprocal of the optimal transportation distance between $\mu,\nu$ when the ground metric on $V$ is taken to be shortest-path distance weighted by $w^{-1}$. 
    \end{corollary}
    
    When the weights $w_{ij}$ are conductances or affinities between the nodes and/or their features, the shortest path metric with weights $w_{ij}^{-1}$ becomes an inverse affinity metric. Moreover, we note that since $B_{w^{-1}, 2}(\mu,\nu)^2 = (\mu-\nu)^T L^\dagger (\mu-\nu)$, where $L^\dagger$ is the Moore-Penrose pseudoinverse of $L$, we may relate $\mathcal{C}_2(\mu,\nu)$ to the effective resistance between the measures.
    \begin{corollary}\label{cor:er}
        Assume $G$ is connected and let $\mu,\nu\in\mathsf{P}(V)$ (with $\mu\neq\nu$). Then it holds
            \begin{align}\label{eq:er}
                \mathcal{C}_{2}(\mu,\nu)^2 &= 1 / (\mu-\nu)^T L^\dagger (\mu-\nu).
            \end{align}
        A minimizer $\phi^\ast$ for $\mathcal{C}_{2}(\mu,\nu)$ achieving \cref{eq:er} is 
            \begin{align}\label{eq:p2soln}
                \phi^\ast = \frac{1}{(\mu-\nu)^\top L^\dagger (\mu-\nu)}L^\dagger (\mu-\nu).
            \end{align}
    \end{corollary}
    
\subsection{Robustness of class labels to perturbations}\label{subsec:robustness}

    In this subsection we provide theoretical motivation for robustness of the class predictions obtained from $p$-conductances when $p=2$. The leading factor $\frac{1}{(\mu-\nu)^T L^+ (\mu-\nu)}$ in \cref{eq:p2soln} ensures that $\phi^\ast(\mu-\nu)=1$ and is always strictly positive when the underlying graph is connected. In the binary setting, class predictions $\widehat{y}_i$ can be obtained by writing $\widehat{y}_i = \mathrm{sign}(\phi^\ast)$. The omission of the leading factor $\frac{1}{(\mu-\nu)^T L^+ (\mu-\nu)}$ leads to identical class predictions, and thus for simplicity we study the stability of the vector $L^+(\mu-\nu)$ under label noise, which we model by fixing two vectors $\eta_1, \eta_2\in\mathbb{R}^n$ with $\eta_1, \eta_2\geq 0$ and $\mathbf{1}_n^T\eta_1 = \mathbf{1}_n^T\eta_2$. Then, we write
        \begin{align*}
            \psi &= L^+(\mu-\nu),\\
            \widetilde{\psi}&= L^+(\mu+\eta_1 - \nu-\eta_2) = L^+(\mu-\nu+\eta)
        \end{align*}
    where $\eta=\eta_1-\eta_2\in\mathbb{R}^n$ can therefore be taken to be any mean-zero vector. Note that we have the obvious bound $\|\psi - \widetilde{\psi}\|_2\leq \|\eta\|_2 / \lambda$, where $\lambda>0$ is the smallest positive eigenvalue of $L$.
    
    This bound is sharp when $\eta$ is an eigenvector of $L$, indicating that class predictions from noisy labels remain stable under small perturbations, especially in expander graphs. In practice, we find that applying the diffusion operator $e^{-tL}$ for small $t>0$ mitigates accuracy drop from label noise to some extent (see \cref{fig:corrupted-labels-1}). Theoretically, under assumptions on $\eta$ and $t$, the worst-case error bound can be improved.
    
    \begin{theorem}\label{thm:robustness}
        Let $\mu,\nu\in\mathsf{P}(V)$ be fixed, and let $\psi = L^+(\mu-\nu)$. Let $\eta\in\mathbb{R}^n$ be a fixed vector with $\mathbf{1}_n^T \eta = 0$, and for $t\geq 0$ let $\widetilde{\psi}_t = L^+e^{-tL}(\mu-\nu + \eta)$. Then we have
            \begin{align}\label{eq:improved-bound}
                \|\psi - \widetilde{\psi}_t\|_2 \leq t\|\mu-\nu\|_2 + \lambda^{-1}e^{-t\lambda}\|\eta\|_2,
            \end{align}
        In particular, if $\frac{\|\eta\|_2}{\|\mu-\nu\|_2} > 1$ and $t\in(0, \lambda^{-1} ( \frac{\|\eta\|_2}{\|\mu-\nu\|_2} - 1) )$, the worst-case error is improved, i.e.,
            \begin{align*}
                t\|\mu-\nu\|_2 + \lambda^{-1}e^{-t\lambda}\|\eta\|_2 < \|\eta\|_2/\lambda.
            \end{align*}
    \end{theorem}

    \begin{remark}
        The condition $\frac{\|\eta\|_2}{\|\mu-\nu\|_2} > 1$ can be interpreted as asking that the perturbation $\eta$ not be too small in order for the diffusion step to have a boosting effect (since the diffusion step inevitably blurs the clean labels as well). For example, let $\mu,\nu$ correspond to disjoint sets of labeled nodes of equal size $m \geq 1$, say $A, B\subseteq V$; i.e.,
            \begin{align*}
                \mu = \frac{1}{m}\mathbf{1}_A,\quad \nu = \frac{1}{m}\mathbf{1}_B.
            \end{align*} 
        Suppose we take $m_1, m_2 < m$ labeled nodes, given by sets $A'\subseteq A, B'\subseteq B$ and corrupt them, setting
            \begin{align*}
                \widetilde{\mu} = \mu - \frac{1}{m}\mathbf{1}_{A'} + \frac{1}{m}\mathbf{1}_{B'},\quad \widetilde{\nu} = \nu + \frac{1}{m}\mathbf{1}_{A'} - \frac{1}{m}\mathbf{1}_{B'}.
            \end{align*}
        Then we have $\eta = 2(\frac{1}{m}\mathbf{1}_{B'} - \frac{1}{m}\mathbf{1}_{A'})$ which satisfies $\|\eta\| = \frac{2}{m}\sqrt{m_1+m_2}$, and thus $\frac{\|\eta\|}{\|\mu-\nu\|} > 1$ holds if and only if $m_1 + m_2 > \frac{1}{2}m$; that is, more than one quarter of the labels are corrupted. If this is the case, \cref{thm:robustness} guarantees that the worst-case $\ell_2$ distance between the clean potential and the potential associated with the corrupted data shrinks with the application of the diffusion operator.
    \end{remark}

\section{Algorithms for general $p \in [1, \infty]$}\label{sec:algorithms}
    
    In this section we present a semismooth Newton augmented Lagrangian algorithm for minimizing \eqref{eq:bcut-intro} for $p \in [1,\infty]$. We note that, as in \cref{cor:er}, solutions to $\mathcal{C}_{2}(\mu,\nu)$ may be obtained directly in terms of the pseudoinverse of $L$. Without loss of generality, we assume the first $m$ vertices $l = \{v_1, v_2,\ldots , v_m \}$ are assigned labels $\{y_1, y_2,\ldots , y_m\}$, where $0 < m \ll n$. In the context of $k$-class classification we take each $y_i$ to be one of the $k$ standard basis vectors $\{e_1, e_2,\ldots , e_k\}$ of the form $e_i = (0,\ldots 0,1,0,\ldots, 0)$, i.e. a one-hot row vector.  Let $M$ denote the number of unlabeled vertices, i.e. $M = n-m$.

    \begin{remark}\label{rmk:one-vs-all}
        The $k$-cut setting is similar to \eqref{eq:bcut-intro}. We define the following ``one-vs-all'' formulation of SSL with labels $\mu_1, \ldots, \mu_k$. Let $R_i = \mu_i - \frac{1}{k-1}\sum_{j\neq i} \mu_j$. We consider a constraint of the form $\text{diag}(\Phi^\top R) = \mathbf{1}_n$. Therefore our framework is, for $1\leq p<\infty$,
            \begin{align}\label{eq:kcut}
                \bar{\mathcal{C}}_p(\mu_1,\ldots, & \mu_k ;  R)^p = \notag\\
                & \min_{\substack{\Phi \in \mathbb{R}^{n\times k} \\ \text{diag}(\Phi^\top R) = \mathbf{1}_k }}  \sum_{\{i, j\} \in E}w_{ij}\|\Phi_i - \Phi_j\|_p^p,
            \end{align}
        and for $p=\infty$, we set
            \begin{align}\label{eq:kcut-2}
                \bar{\mathcal{C}}_\infty(\mu_1,\ldots, &\mu_k ; R)= \notag\\
                & \min_{\substack{\Phi \in \mathbb{R}^{n\times k} \\ \text{diag}(\Phi^\top R) = \mathbf{1}_k }}  \max_{\{i, j\} \in E} w_{ij}\|\Phi_i - \Phi_j\|_\infty.
            \end{align}
    \end{remark}
    
\subsection{Duality and optimality conditions}\label{subsec:duality-conditions}
    
    In this section we compute the Lagrangian dual to \eqref{eq:bcut-intro} and establish optimality conditions. Let $p\in[1,\infty]$ be fixed. First, we introduce the variable $u\in\mathbb{R}^{2N}$ with entries indexed by the oriented edges $E'$, and express the primal problem \eqref{eq:bcut-intro} in the following equivalent form
        \begin{align}\tag{Pr}\label{eq:bcut-primal}
            \min_{\phi \in \mathbb{R}^n, u \in \mathbb{R}^{N}} s(u)  \quad \text{s.t. } &\phi^\top (\mu - \nu) = 1 \\ 
            &B^\top \phi - u = 0\notag,
        \end{align}
    where $s(u) = \sum_{\{i, j\} \in E}w_{ij}|u_{ij}|^p$ for $1\leq p <\infty$ and with conventional modification for $p=\infty$, and where (with slight abuse) $B\in\mathbb{R}^{n\times 2N}$ is the node-edge oriented incidence matrix with respect to $E'$.  The dual problem is
        \begin{align}\tag{Du}\label{eq:bcut-dual}
            \min_{\substack{y\in\mathbb{R}\\ z\in\mathbb{R}^{N}}} s^*(u) - y \quad \text{s.t. } y(\mu - \nu) + B z = 0,
        \end{align}
    where $s^*$ is the Fenchel conjugate of $s(u)$. Now, denote by $\mathcal{L}$ the Lagrangian function for~\eqref{eq:bcut-primal}:
        \begin{align*}
            \mathcal{L}(\phi, u;y, z) = s(u) &+ \langle y, \phi^\top (\mu - \nu) - 1 \rangle
            + \langle z, B^\top \phi - u  \rangle.
        \end{align*}
    Furthermore, given $\sigma_1,\sigma_2 > 0$, define the augmented Lagrangian function associated with~\eqref{eq:bcut-primal}:
        \begin{align}\label{eq:augmented_lagrangian}
            \mathcal{L}_\sigma (\phi, u;y, z) = \mathcal{L}(\phi, u;y, z) &+ \frac{\sigma_1}{2}||B^\top\phi - u||^2 \notag\\
            &+ \frac{\sigma_2}{2}||\phi^\top (\mu - \nu) - 1||^2.
        \end{align}
    The KKT conditions for~\eqref{eq:bcut-primal} and~\eqref{eq:bcut-dual} are
        \begin{align}\tag{KKT}
            \begin{cases}
                y(\mu - \nu) + Bz&= 0\\
                z &\in \partial s(u) \\
                \phi^\top (\mu - \nu) - 1 &= 0 \\
                B^\top \phi - u &= 0.
            \end{cases}
        \end{align}
    Note that the condition $0 \in \partial_u \mathcal{L} = \partial s(u) -z \implies z\in\partial s(u)$ can be expressed using the proximal map:
        \begin{align*}
            \mathrm{prox}_{\lambda s}(v) = \argmin_u \left\{ \frac{1}{2}||u - v||^2 +\lambda s(u) \right\}.
        \end{align*}
    An equivalent characterization of the proximal operator is $0 \in \lambda \partial s(u) + u - v$. i.e., $u = v-\lambda z$ with $u \in \partial s(u)$ and the stationarity condition can be written
        \begin{align*}
            u - \mathrm{prox}_{\lambda s}(B^\top\phi + \frac{1}{\sigma_1} z) = 0.
        \end{align*}
    
\subsection{Semismooth Newton-CG augmented Lagrangian}\label{subsec:ssncg}

    The inexact ALM method below is a well known framework for solving convex composite optimization problems, and seeks solutions of~\eqref{eq:bcut-primal} by searching for saddle points of~\eqref{eq:augmented_lagrangian}:
        \begin{align*}
            \min_{\phi, u} \max_{y, z} \mathcal{L}_\sigma(\phi, u; y,z).
        \end{align*}
    \begin{algorithm}[H]\caption{Semismooth augmented Lagrangian (SSNAL)}\label{alg:alg1}
        \begin{algorithmic}[1]\label{alg:admm}
            \WHILE{not converged} 
            \STATE Compute \\
            $(\phi_{k+1}, u_{k+1}) \approx \argmin_{\phi, u} \mathcal{L}_\sigma (\phi, u; z_k, y_k)$
            \STATE Compute $z_{k+1} = z_k + (\sigma_1)_k (B^\top \phi_{k+1} - u_{k+1})$
            \STATE Compute $y_{k+1} = y_k + (\sigma_2)_k (\phi^\top (\mu - \nu) - 1)$
            \STATE Update $\sigma_{k+1} \leq \infty$
            \ENDWHILE
        \end{algorithmic}
    \end{algorithm}

    We design a semismooth Newton-CG algorithm to solve the expression in line 2 as follows. Let $p\in[1,\infty]$ be fixed. For a given $\sigma_1$ and $\sigma_2$ and $\tilde{z}$ and $\tilde{y}$, the subproblem is
        \begin{align}\label{eq:subproblem}
            \min_{\substack{\phi\in\mathbb{R}^n\\u\in\mathbb{R}^{N}}}\left\{ F(\phi, u) := \mathcal{L}_\sigma (\phi, u; \tilde{z}, \tilde{y}) \right\}. 
        \end{align}
    Since $F(\phi, u)$ is a strongly convex function, the level set $\{(\phi, u)|F(\phi, u) \leq \eta\}$ is a closed and bounded convex set for any $\eta \in \mathbb{R}$. Hence \eqref{eq:subproblem} admits a unique optimal solution which we denote as $(\bar{\phi}, \bar{u})$. Now, for any $\phi$, denote
        \begin{align*}
            f(\phi) &:= \inf_u F(\phi, u) \\
            &= \inf_u \left\{ s(u) + \frac{\sigma_1}{2}||u - (B^\top \phi + \frac{1}{\sigma_1}\tilde{z})||^2\right\} \\
            &\quad- \frac{1}{2\sigma_1}||\tilde{z}||^2 + \tilde{z}^\top B^\top \phi + \langle \tilde{y}, \phi^\top (\mu - \nu) - 1 \rangle \\
            &\quad+ \frac{\sigma_2}{2}||\phi^\top (\mu - \nu) - 1||^2
        \end{align*}
    We can compute $(\bar{\phi}, \bar{u}) = \argmin F(\phi, u)$ via
        \begin{align*}
            \bar{\phi} = \argmin f(\phi), \quad \bar{u} = \mathrm{prox}_{s/\sigma_1}(B^\top \bar{\phi} + \frac{1}{\sigma_1}\tilde{z})
        \end{align*}
    $f$ is strongly convex and continuously differentiable with
        \begin{align*}
            \nabla f(\phi) = &B\tilde{z} + \sigma_1 B(B^\top 
            \phi - \bar{u}) + \tilde{y}(\mu - \nu) \\
            &+ \sigma_2(\phi^\top (\mu - \nu) - 1)(\mu - \nu),    
        \end{align*}
    so $\bar{\phi}$ is the solution to the following nonsmooth equation:
        \begin{align}\label{eq:newtoncgsystem}
            \nabla f(\phi) = 0.
        \end{align}
    For any given $\phi \in \mathbb{R}^{n}$, we define
        \begin{align*}
            \hat{\partial}^2f(\phi) := \sigma_1 BB^\top - \sigma_1B\partial\bar{u}(\phi) + \sigma_2 (\mu - \nu)(\mu - \nu)^\top.
        \end{align*}
    Here, $\partial\bar{u}(\phi) = \partial \mathrm{prox}_{s/\sigma_1}(B^\top \bar{\phi} + \frac{1}{\sigma_1}\tilde{z})\cdot B^\top$ is the generalized Jacobian of $\bar{u}(\phi) = \mathrm{prox}_{s/\sigma_1}(B^\top \bar{\phi} + \frac{1}{\sigma_1}\tilde{z})$ composed with $B^\top$. Recall that if $u = \mathrm{prox}_{\lambda s}(v)$, we have $\lambda^{-1}(v-u) \in \partial s(u)$. This identity relates the subgradient of $s(u)$ to the prox operator. Thus, understanding $\partial\bar{u}(\phi)$ reduces to understanding the behavior of $\partial s(u)$ and the mapping $\mathrm{prox}_{\lambda s}(\cdot)$. We can thus consider two cases separately: 

    \noindent \textbf{Case $\mathbf{(p \in [1,\infty))}$:}  $s(u)$ is strictly convex and differentiable at all points where $u_i \neq 0$. The generalized Jacobian $\partial \bar{u}(\phi)$ simplifies to a diagonal weighting $D_{\bar{u}}$, with entries
        \begin{align*}
            (D_{\bar{u}})_{ii} = \frac{1}{1 + \lambda p w_i(p-1)|\bar{u}_u|^{p-2}}.
        \end{align*}

    \noindent \textbf{Case $\mathbf{(p = \infty)}$:} The subdifferential depends on the set of indices that achieve the maximum magnitude. Let
        \begin{align*}
            M = ||u||_\infty \text{ and } A = \{i : |u_i| = M\}.
        \end{align*}
    Then the subdifferential is the convex set
        \begin{align*}
            \{z \in \mathbb{R}^n : z_i = \theta_i,\:\: |\theta_i| \leq 1, \:\: \theta_i = \text{sign}(u_i) \text{ if } i \in A\}.
        \end{align*}

    \begin{remark}\label{rem:posdefhessian}
        If the underlying graph $G$ is connected, $H = \sigma_1(BB^\top - B D_{\bar{u}}B^\top) + \sigma_2(\mu-\nu)(\mu - \nu)^\top$ is symmetric and positive definite.
    \end{remark}
    
    Based on the preceding setup, we can thus present the following semismooth Newton-CG algorithm for solving \eqref{eq:newtoncgsystem}.

    \begin{algorithm}[H]\caption{Semismooth Newton conjugate gradient}\label{alg:alg2}
        \begin{algorithmic}[1]
            \WHILE{$||\nabla f(\phi)|| \leq \epsilon$} 
            \STATE Compute an approximate Newton step, $\Delta\phi$ \\
            $$H \Delta\phi \approx -\nabla f(\phi)$$
            \STATE (Armijo rule given $r > 0$, $\sigma,\eta \in (0,1)$) Set $\zeta = \eta^{m}r$, where $m$ is the first nonnegative integer s.t.
            $$
            f(\phi_t) - f(\phi_t + \eta^{m}r\Delta \phi_t) \geq \sigma \eta_ts \langle \nabla f(\phi_t), \Delta \phi_t \rangle
            $$
            \STATE Set $\phi_{t+1} = \phi_t + \zeta \Delta \phi$
            \ENDWHILE
        \end{algorithmic}
    \end{algorithm}
    
\subsection{Analysis of SSNAL}\label{subsec:analysis}
    Due to remark \ref{rem:posdefhessian}, the conjugate gradient procedure is well-defined and a natural choice for solving the Newton system associated with \eqref{eq:newtoncgsystem}. The computational cost for conjugate gradient methods is highly dependent on the cost for computing the matrix-vector product $Hx$ for any given $x \in \mathbb{R}^d$. 
    \begin{remark}
        Consider the computation $Hx$ for $x \in \mathbb{R}^n$. The cost is $O(|E|)$ arithmetic operations. 
    \end{remark}
    In other words, the sparsity structure of the underlying graph characterizes the complexity of computing the intermediate updates. To ensure the convergence of SSNAL, the convergence of the semismooth newton Method and the augmented Lagrangian method must be established.
        
    \begin{definition}[Semismoothness]
        Let $F : \mathcal{O} \subseteq \mathcal{X} \to \mathcal{Y}$ be a locally Lipschitz function on an open set $\mathcal{O}$. $F$ is said to be semismooth at $x \in \mathcal{O}$ if $F$ is directionally differentiable at $x$ and for any $V \in \partial F(x + \Delta x)$ with $\Delta x \to 0$
        $$
        F(x + \Delta x) - F(x) - V\Delta x = O(||\Delta x||)
        $$
        $F$ is strongly semismooth at $x$ if $F$ is semismooth at $x$ and
        $$
        F(x + \Delta x) - F(x) - V\Delta x = O(||\Delta x||^2)
        $$
        $F$ is said to be a semismooth (respectively, strongly semismooth) function on $\mathcal{O}$ if it is semismooth (respectively, strongly semismooth) everywhere in $\mathcal{O}$.
    \end{definition}
    \begin{proposition}[Semismoothness of the proximal map]\label{prop:prox-semismooth}
        For any $t > 0$, and integer values of $p\in[0,\infty)$ the proximal mapping $\mathrm{prox}_{t s(\cdot)}$ is strongly semismooth.
    \end{proposition}
    The result is proved in Appendix \ref{sec:proofs-sec-3}. Convergence of inexact ALM is ensured whenever the following criterion holds, for a summable sequence of nonnegative numbers $\{\epsilon_k\}$.
    \begin{equation}\label{eq:stopping-criterion}
        \text{dist}(0, \partial F(\phi_{k+1}, u_{k+1})) \leq \epsilon_{k}/\max\{1,\sqrt{\sigma_k}\}
    \end{equation}
    \begin{theorem}[Convergence of SSNAL]\label{thm:ssnal-convergence}
        Let ($\phi_k, u_k, z_k, y_k$) be the sequence generated by SSNAL with stopping criterion \eqref{eq:stopping-criterion}. Then, the sequence $\{\phi_k\}$ is bounded and converges to the unique optimal solution of \eqref{eq:bcut-primal}, and $||B^\top \phi - u||$ and $||\phi^\top (\mu - \nu) - 1||$ converge to zero. 
    \end{theorem}
    The above convergence theorem follows from~\cite{rockafellar1976a, rockafellar1976b}. Next, we state the convergence property for the semismooth Newton method used to solve the subproblems associated with \eqref{eq:newtoncgsystem}. The proof is provided in Appendix \ref{sec:proofs-sec-3}.
    
    \begin{theorem}[Convergence of SSNCG]
        Assume that $\text{prox}_{\lambda s}(\cdot)$ is strongly semismooth on $int(dom(s))$. Let the sequence $\{\phi_k\}$ be generated SSNCG.Then, $\{\phi_k\}$ converges to the unique solution $\bar{\phi}$ of \eqref{eq:newtoncgsystem} and 
        $$
        ||\phi_{k+1} - \bar{\phi}|| = O(||\phi_k - \bar{\phi})||^{1+\tau})
        $$
        where $\tau \in (0,1]$ is a given constant in the algorithm.
    \end{theorem}

\section{Experimental results}\label{sec:experimental-results}

    \begin{figure}[h!]
        \centering \includegraphics[width=0.5\linewidth]{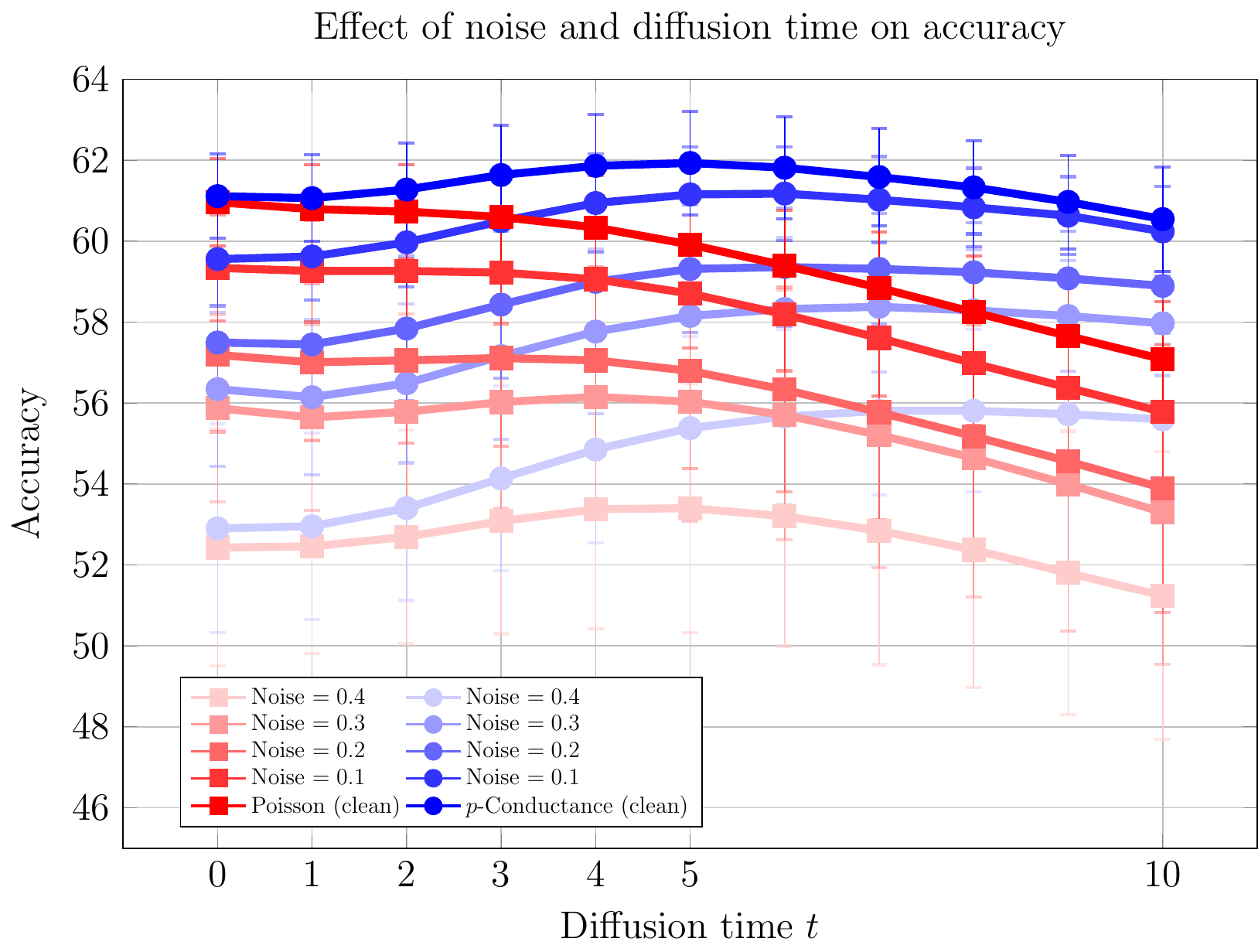}
        \caption{Accuracy on CIFAR-10 for various  diffusion times and robustness. Shades denote noise level. Label rate is 10 labels per class.}
        \label{fig:corrupted-labels-1}
    \end{figure}

    \begin{table*}[t!]
    \caption{Average accuracy over 100 trials with standard deviation in brackets. Best is bolded.}\label{tab:cora-pubmed}\vspace*{-.5cm}
    \begin{center}
    \begin{small}
    \begin{sc}
    \begin{tabular}{llllll}
    \toprule
    Cora \# Labels per class&\textbf{1}&\textbf{3}&\textbf{5}&\textbf{10}&\textbf{100}\\
    \midrule
    Laplace/LP \cite{zhu2003semi}&   21.8 (14.3)    & 37.6 (12.3)     &  51.3 (11.9)      & 66.9 (6.8)   &  81.8 (1.1)       \\
    Sparse LP \cite{jungsparselp2016} & 16.0 (1.8) & 19.4 (1.8) & 23.1 (2.3) & 28.7 (2.2) & 47.0 (2.2) \\
    p-laplace \cite{flores2022}&   41.9 (8.9)    & 57.6 (7.1)     &  61.9 (6.2)      & 67.9 (3.6)   &  79.2 (1.3)       \\
    p-eikonal \cite{calder2022hamilton}&   40.4 (8.6)    &  51.8 (7.2)    &   56.9 (6.8)     &  64.5 (4.2)  & 80.2 (1.1) \\       
    Poisson \cite{calder20poisson}      &   57.4 (9.2)   &  67.0 (4.9)      &    69.3 (3.6)     & 71.6 (3.0)   &  76.0 (1.0)    \\
    $p$-conductance ($p=1$, $\epsilon=n$)       & 22.9 (6.8)   &  22.7 (6.8)  &  21.4 (6.6) & 21.3 (6.6)    & 35.6 (13.2)  \\
    $p$-conductance ($p=2$, $\epsilon=n$)       &  58.9 (7.2) & 67.9 (3.7)   & 70.2 (2.3)  & 72.2 (1.9)   & 75.2 (1.3)  \\
    $p$-conductance ($p=\infty$, $\epsilon=n$)       & 44.3 (6.8)   &  53.7 (5.0)  &  58.9 (3.6) & 63.7 (3.1)    & 73.7 (1.4)  \\
[0.5ex]
    \cdashline{1-6}\noalign{\vskip 0.5ex}
    PoissonMBO \cite{calder20poisson}      &   58.5 (9.4)    & 68.5 (4.1)       &   70.7 (3.0)     &  73.3 (2.3)  &  80.1 (0.9)  \\
    \textbf{$p$-conductance ($p=2$, $\epsilon=0$)}       &  \textbf{63.1 (8.0)} & \textbf{72.9 (3.5)}   & \textbf{75.5 (1.8)}  & \textbf{77.9 (1.1)}   & \textbf{82.9 (0.9)}  \\
    \bottomrule
    \toprule
    Pubmed \# Labels per class&\textbf{1}&\textbf{3}&\textbf{5}&\textbf{10}&\textbf{100}\\
    \midrule
    Laplace/LP \cite{zhu2003semi}&   34.6 (8.8)    &  35.7 (8.2)    &  36.9 (8.1)      &  39.6 (9.1)  &  74.9 (3.6)       \\
    Sparse LP \cite{jungsparselp2016} & 32.4 (4.7) & 33.0 (4.8) & 33.6 (4.8) & 33.9 (4.8) & 43.2 (4.1) \\
    p-laplace \cite{flores2022}&   44.8 (11.2)    & 58.3 (9.1)     &  61.6 (7.7)      & 66.2 (4.7)   &  74.3 (1.1)       \\
    p-eikonal \cite{calder2022hamilton}& 44.3 (11.8)      &  55.6 (10.0)    &   58.4 (9.1)     &  65.1 (5.8)  & 74.9 (1.5) \\       
    Poisson \cite{calder20poisson}      &  55.1 (11.3)    &  66.6 (7.4)      &   68.8 (5.6)     &  71.3 (2.2)  &  75.7 (0.8)    \\
    $p$-conductance ($p=1$, $\epsilon=n$)       & 39.6 (0.3)  &  39.6 (0.3) & 39.6 (0.3)  &  40.3 (0.3)  & 41.2 (0.3)  \\
    $p$-conductance ($p=2$, $\epsilon=n$)       &  58.0 (12.1)  &  67.5 (7.5)  &  70.8 (4.9) & \textbf{72.4 (2.5)}    & 77.6 (0.6)  \\
    $p$-conductance ($p=\infty$, $\epsilon=n$)       & 48.0 (8.1)  & 56.5 (5.2)  & 57.0 (8.2)  &  62.6 (3.0)  & 72.9 (1.3)  \\[0.5ex]
    \cdashline{1-6}\noalign{\vskip 0.5ex}
    PoissonMBO \cite{calder20poisson}      &    54.9 (11.4)   &    65.3 (7.8)     &    68.2 (5.3)    &  69.9 (3.0)  &74.8 (1.0)   \\
    \textbf{$p$-conductance ($p=2$, $\epsilon=0$)}       &  \textbf{58.7 (11.5)}  &  \textbf{67.5 (7.5)}  &  \textbf{70.8 (4.8)} & \textbf{72.4 (2.6)}    & \textbf{77.8 (0.6)}  \\
    \bottomrule
    \end{tabular}
    \end{sc}
    \end{small}
    \end{center}\vspace{-.25cm}
    \end{table*}
    
    We represent the initial label configuration by a matrix 
    $Y\in\mathbb{R}^{n\times k}$. If node $i$ is labeled by label $\ell$, we view $Y_{i,\ell}$ as placing a Dirac measure at node~$i$ with mass concentrated on label~$\ell$. We diffuse these measures according to the heat equation on the graph via the heat kernel $e^{-tL}$ with time parameter $t>0$. The diffused measures are
        \begin{equation}
            Y_{\text{diffused}} = e^{-tL}Y,
        \end{equation}
    since $e^{-tL}$ is a mass-preserving operator (specifically, if $\mu\in\mathbb{R}^n$ is a probability measure, so is $e^{-tL}\mu$), and thus yields a ``smoothed'' measure $Y_{\text{diffused}}$. We then use $Y_{\text{diffused}}$ to construct the matrix $R$ (as described in Remark \eqref{rmk:one-vs-all}) and solve \crefrange{eq:kcut}{eq:kcut-2}. Diffusing labeled measures improves accuracy in both the clean and noisy label settings.

    \begin{remark}\label{rmk:class-pred}
        Note that the solutions to \crefrange{eq:kcut}{eq:kcut-2} typically take on continuous values which then need to be mapped to class predictions. In \cref{sec:cut-assignment}, we propose a heuristic that is capable of exploiting estimates of the class cardinalities. In the limiting case when no estimates of class cardinalities are provided, predictions reduce to the heuristic of selecting the class (column) with the largest magnitude:
             \begin{equation}\label{eq:labeldec}
                 y^\text{pred}_i = \argmax_{j\in \{1,\dots,k\}} \{\phi^\ast_{ij}\}.
             \end{equation}
        The deviation of the estimate from the ground-truth cardinalities is governed by the parameter $\epsilon$. Thus, we denote by 
        \textbf{$p$-conductance} ($\epsilon=0$) to be our method without any cardinality prior and \textbf{$p$-conductance} ($\epsilon=n$) to be our method with exact knowledge of class cardinalities.
    \end{remark}

    We compare to the state-of-the-art unconstrained graph-based SSL algorithm Poisson Learning \cite{calder20poisson} as well as variants based on the MBO procedure \cite{jacobs2018auction, calder20poisson} to incorporate knowledge of class sizes. We also include classifiers derived from the $p$-Laplacian~\cite{flores2018algorithms}, the $p$-eikonal equations~\cite{calder2022hamilton}, and TV-distance~\cite{jungsparselp2016}.

    \noindent\textbf{Citation networks} In \cref{tab:cora-pubmed} and \cref{tab:citeseer}, we evaluate our method on the Planetoid dataset~\cite{planetoid} corresponding to citation networks, where the graph is given and not data-driven. In these networks, the vertices represent publications and their links refer to citations between publications, while the label corresponds to the subject of the document. Note that we evaluate on the largest connected component of each dataset. We omit results on VolumeMBO due to runtime errors when evaluating the method on this dataset. For Cora, we outperform PoissonMBO by $4.6\%$ at 1 label per class, and trend persists across datasets and label rates. We also include in \cref{fig:illustration-cuts-cora} (see \cref{sec:experiments}) a visualization of the solution $\phi^\ast$ as in \cref{cor:er} in the setting of the one-vs-all measure minimum cut for $p=2$.
        
    \noindent\textbf{Image datasets}  We evaluate our method on three image datasets: MNIST \cite{mnist}, Fashion-MNIST \cite{fmnist} and CIFAR-10 \cite{cifar}, using pretrained autoencoders as feature extractors as in \cite{calder20poisson}. We include more details in \cref{sec:experiments}, and include results in \cref{tab:mnist}, \cref{tab:fmnist}, and \cref{tab:cifar}. $p$-conductance learning outperforms all other methods. Notably, with cardinality estimates on CIFAR-10, our method outperforms PoissonMBO by 2.4\%.
    
    \noindent\textbf{Corrupted labels} We also present results demonstrating accuracy under random label corruptions. In this setting, the labels associated with a fixed proportion of the training set are randomly flipped to an incorrect class. We show that our method exhibits superior robustness, with a 3.8\% improvement in accuracy compared to Poisson learning on CIFAR-10 when 40\% of the labels have been flipped. In Figure \ref{fig:corrupted-labels-1}, we show that this robustness is dependent on the diffusion of the labels. We observe that Poisson learning is not able to exploit the diffused labels as effectively. Note that methods which enforce interpolation of the labels (e.g. Laplace learning) may not be applied to the diffused labels due to the propagation of label information to \emph{every} node. 
    
    \noindent\textbf{Partially labeled learning (PLL)} PLL is a weakly supervised learning problem where each training instance is equipped with a \emph{set} of candidate labels, including the ground-truth label. In the PLL setting, the labels associated with a fixed proportion of the training set are augmented with a set of chosen candidate classes. In this experiment, we consider the superclasses associated with the CIFAR-100~\cite{cifar} dataset in addition to 100 fine-grained class labels. We consider a random \emph{subset} of fine-grained labels associated with each training image as being given. The task is to learn the fine-grained classes for all test images. We demonstrate superior performance, with $p$-conductance-MBO outperforming PoissonMBO by 4.1\% with a partial label set size of 4. Interestingly, $p$-conductance learning with $p=\infty$ performs best with partial label set size of 5. our results are shown in \cref{tab:pll}.

\section{Conclusion}

    This work presented $p$-conductance learning, a robust semi-supervised learning method. By framing the problem in terms of measure minimum cuts, we obtain connections to effective resistance and optimal transport, in addition to results which support the method’s strong theoretical and empirical performance. Extensive experiments validate its capacity to handle scarce or corrupted labels while achieving state-of-the-art results across multiple datasets. Possible future directions include:
        \begin{itemize}
            \item The development of active label selection strategies in tandem with the $p$-conductance learning framework;
            \item Theoretical investigations into the statistical consistency of our method;
            \item Exploration of SSL frameworks which are robust to data corruption in addition to label corruption.
        \end{itemize}

\section*{Acknowledgements}

SR was supported by the Halicio\u{g}lu Data Science Institute graduate prize fellowship. GM was supported by NSF CCF-2217058. GM and CH were supported by an award from the W.M. Keck Foundation.

\appendix

\section{Additional experiments and figures}\label{sec:experiments}

\subsection{Image datasets}

    In this section, we provide experimental results on MNIST, Fashion-MNIST, and CIFAR-10. We demonstrate our method on a range of semi-supervised learning problems, using the \texttt{graphlearning} package~\cite{graphlearning}.

    For MNIST and Fashion-MNIST, we used variational autoencoders with 3 fully connected layers of sizes (784,400,20) and (784,400,30), respectively, followed by a symmetrically defined decoder. The autoencoder was trained for 100 epochs on each dataset. The autoencoder architecture, loss, and training are similar to \cite{kingma2013bayes}. 

    \begin{table*}[htb!]
        \caption{MNIST. Average accuracy over 100 trials with standard deviation in brackets.}\vspace*{-.25cm}
        \label{tab:mnist}
        \begin{center}
            \begin{small}
                \begin{sc}
                    \begin{tabular}{llllll}
                        \toprule
                        MNIST \# Labels per class & \textbf{1} & \textbf{3} & \textbf{5} & \textbf{10} & \textbf{4000} \\
                        \midrule
                        Label propagation \cite{zhu2003semi} & 16.1 (6.2) &  42.0 (12.4) & 69.5 (12.2) & 94.8 (2.1) & 98.3 (0.0) \\
                        Sparse LP \cite{jungsparselp2016}       &  14.0 (5.5) &  14.5 (4.0) & 16.2 (4.2) & 18.4 (4.1) & 94.1 (0.1) \\
                        p-laplace \cite{flores2022}             & 83.2 (2.4) & 87.4 (1.5) & 90.7 (0.9) & 94.8 (0.1) & \textbf{98.1 (0.1)} \\
                        $p$-eikonal \cite{calder2022hamilton}      &   81.4 (4.5)    &   91.2 (1.3)    &   92.5 (1.1)    & 94.3 (0.6)  & \textbf{98.1 (0.1)}    \\
                        Poisson \cite{calder20poisson}           & 95.8 (0.8) &96.3 (0.3)  & 96.7 (0.1) & 97.0 (0.1) & 97.2 (0.1) \\
                        $p$-conductance ($p=1$, $\epsilon=n$)        & 24.3 (1.6)  & 31.8 (1.4)  & 43.1 (1.1)  & 53.9 (0.8)  & 96.3 (0.1) \\
                        $p$-conductance ($p=2$, $\epsilon=n$)        & 95.4 (1.3) & 96.4 (0.5) & 96.7 (0.3) & 97.2 (0.1) & 97.3 (0.1) \\
                        $p$-conductance ($p=\infty$, $\epsilon=n$)    & 86.3 (0.8)  & 89.2 (0.6) & 92.7 (0.4) &  95.3 (0.1) & 97.2 (0.1) \\[0.5ex]
                        \cdashline{1-6}\noalign{\vskip 0.5ex}
                        VolumeMBO \cite{jacobs2018auction}      &    89.9 (7.3)    &   96.2 (1.2)     &    96.7 (0.6)    &  96.7 (0.2)   & 96.9 (0.1)     \\
                        PoissonMBO \cite{calder20poisson}        & \textbf{97.5 (0.1)} & \textbf{97.5 (0.1)} & \textbf{97.5 (0.1)} & \textbf{97.6 (0.1)} & \textbf{98.1 (0.1)} \\
                        \textbf{$p$-conductance ($p=2$, $\epsilon=0$)}             & 96.6 (0.4) & 96.9 (0.2) & 97.1 (0.1) & 97.4 (0.1) & 97.5 (0.1) \\
                        \textbf{$p$-conductance-MBO}             & \textbf{97.5 (0.1)} & \textbf{97.5 (0.1)} & \textbf{97.5 (0.1)} & \textbf{97.6 (0.1)} & \textbf{98.1 (0.1)}  \\
                        \bottomrule
                    \end{tabular}
                \end{sc}
            \end{small}
        \end{center}
    \end{table*}
    \begin{table*}[htb!]
        \caption{FashionMNIST. Average accuracy over 100 trials with standard deviation in brackets.}\vspace*{-.25cm}
        \label{tab:fmnist}
        \begin{center}
            \begin{small}
                \begin{sc}
                    \begin{tabular}{llllll}
                        \toprule
                        FashionMNIST \# Labels per class & \textbf{1} & \textbf{3} & \textbf{5} & \textbf{10} & \textbf{4000} \\
                        \midrule
                        Label propagation \cite{zhu2003semi} &  18.4 (7.3) &  44.0 (8.6) &  57.9 (6.7) & 70.6 (3.1) &  85.8 (0.1) \\
                        Sparse LP \cite{jungsparselp2016}       &  14.1 (3.8) & 13.7 (3.3) &   16.1 (2.5) &  15.2 (2.5)  & 79.3 (0.1) \\
                        p-laplace \cite{flores2022}             & 60.9 (3.0) & 65.2 (2.5) & 69.4 (1.5) & 78.0 (0.4) & 85.5 (0.1) \\
                        $p$-eikonal \cite{calder2022hamilton}      &   55.1 (4.4)    &   61.3 (3.0)    &   65.5 (3.4)    & 70.1 (1.2)  & 85.8 (0.1)    \\
                        Poisson \cite{calder20poisson}           & 68.6 (2.6) & 71.7 (2.3) & 74.3 (1.5) & 79.6 (0.6) & 81.1 (0.1) \\
                        $p$-conductance ($p=1$, $\epsilon=n$)        & 18.2 (6.3) & 21.3 (5.1) & 29.3 (3.9)  & 46.4 (2.4)  & 80.4 (0.3)  \\
                        $p$-conductance ($p=2$, $\epsilon=n$)        & 66.5 (2.8) & 69.9 (2.5) & 73.3 (1.7) & 80.2 (0.6) & 81.4 (0.2) \\
                        $p$-conductance ($p=\infty$, $\epsilon=n$)    &  54.2 (2.6) & 58.3 (1.9) & 64.1 (2.1)  & 72.8 (0.9) & 80.8 (0.1)  \\[0.5ex]
                        \cdashline{1-6}\noalign{\vskip 0.5ex}
                        VolumeMBO \cite{jacobs2018auction}      &    54.7 (5.2)    &    66.1 (3.3)    &  70.1 (2.8)     &   74.4 (1.5)   &  75.1 (0.2)    \\
                        PoissonMBO \cite{calder20poisson}        & 69.7 (2.8) & 73.0 (2.3) & 75.7 (1.6) & 80.9 (0.5) & 86.3 (0.1) \\
                        \textbf{$p$-conductance ($p=2$, $\epsilon=0$)}             & 70.2 (2.5)  & \textbf{73.4 (1.9)} & \textbf{76.1 (1.4)} & 81.0 (0.4) & \textbf{87.3 (0.1)} \\
                        \textbf{$p$-conductance-MBO}             &  \textbf{71.6 (0.7)} & 72.5 (1.5) & \textbf{76.1 (0.8)} & \textbf{82.4 (0.5)} & 86.5 (0.1) \\
                        \bottomrule
                    \end{tabular}
                \end{sc}
            \end{small}
        \end{center}
    \end{table*}
    \begin{table*}[htb!]
        \caption{CIFAR-10. Average accuracy over 100 trials with standard deviation in brackets.}\vspace*{-.25cm}
        \label{tab:cifar}
        \begin{center}
            \begin{small}
                \begin{sc}
                    \begin{tabular}{llllll}
                        \toprule
                        CIFAR-10 \# Labels per class&\textbf{1}&\textbf{3}&\textbf{5}&\textbf{10}&\textbf{4000}\\
                        \midrule
                        Label propagation \cite{zhu2003semi} & 10.4 (1.3)      &11.6 (2.7)     & 14.1 (5.0)      & 21.8 (7.4) & 80.9 (0.0)       \\
                       Sparse LP \cite{jungsparselp2016} &  11.8 (2.4)  &  11.1 (3.3)  & 14.4 (3.5)  &  15.6 (3.1) & 75.5 (0.3) \\
                        p-laplace \cite{flores2022}&   33.4 (4.7)     &   42.1 (3.1)     &   47.2 (2.7)    &  52.2 (1.6)  &  80.8 (0.2)      \\
                        $p$-eikonal \cite{calder2022hamilton}      &   32.9 (5.1)    &   42.8 (5.4)    &   49.7 (3.5)    & 52.8 (3.0)  & 80.5 (0.1)    \\ 
                        Poisson \cite{calder20poisson}      &   39.3 (5.4)    &   48.9 (3.5)    &   53.3 (2.8)    & 57.8 (1.8)  & 80.3 (0.2)    \\
                        $p$-conductance ($p=1$, $\epsilon=n$)       & 13.4 (6.1)  &  17.1 (3.7)  & 21.8 (2.9)  & 39.6 (2.1)    & 78.4 (0.3)  \\
                        $p$-conductance ($p=2$, $\epsilon=n$)       &  40.0 (5.1) & 50.4 (3.4)   & 54.9 (2.6)  & 59.4 (1.6)   & 80.6 (0.2)  \\
                        $p$-conductance ($p=\infty$, $\epsilon=n$)       &  30.1 (4.7)  & 39.7 (3.1)   & 43.6 (2.2)  &  45.9 (1.4)   & 79.1 (0.2) \\ [0.5ex]
                        \cdashline{1-6}\noalign{\vskip 0.5ex}
                        VolumeMBO \cite{jacobs2018auction}       &  38.0 (7.2)  & 50.1 (5.7)   & 55.3 (3.8)   & 59.2 (3.2)   &  75.8 (0.9)  \\
                        PoissonMBO \cite{calder20poisson}       & 40.8 (6.8)  & 53.1 (4.1)   &  58.1 (3.2) & 62.0 (1.9)   & 80.5 (0.2)  \\
                        \textbf{$p$-conductance ($p=2$, $\epsilon=0$)}       & 40.3 (5.6)   &  52.1 (3.6)  & 57.1 (2.3)  & 61.5 (1.5)  & 80.7 (0.2) \\
                        \textbf{$p$-conductance-MBO}       & \textbf{43.2 (6.5)}   & \textbf{54.0 (3.2)}  & \textbf{59.2 (2.4)}  & \textbf{62.6 (1.0)}  & \textbf{80.8 (0.2)}  \\
                        \bottomrule
                    \end{tabular}
                \end{sc}
            \end{small}
        \end{center}
    \end{table*}

    For each dataset, we construct a graph over the latent feature space. We used all available data to construct the graph, with $n=70,000$ nodes for MNIST and Fashion-MNIST, and $n=60,000$ nodes for CIFAR-10. The graph was constructed as a $k$-nearest neighbor graph with Gaussian edge weights
        \begin{align*}
            w_{ij} =\exp\left( -4||x_i-x_j||^2/d_k(x_i)^2 \right),
        \end{align*}
    where $x_i$ are the latent variables for image $i$, and $d_k(x_i)$ is the distance in the latent space between $x_i$ and its $k^{\rm th}$ nearest neighbor. We used $k=10$ in all experiments and symmetrize $W$. Results are shown in \cref{tab:mnist}, \cref{tab:fmnist}, \cref{tab:cifar}. Notably, with cardinality estimates on CIFAR-10, our method outperforms PoissonMBO by 2.4\%. On FashionMNIST at 1 label rate, $p$-conductance-MBO outperforms PoissonMBO by 1.9 \%

    \begin{table*}[htb!]
        \caption{Average accuracy over 100 trials with standard deviation in brackets. Best is bolded.}\vspace*{-.25cm}
        \label{tab:citeseer}
        \begin{center}
        \begin{small}
        \begin{sc}
        \begin{tabular}{llllll}
        \toprule
        Citeseer \# Labels per class&\textbf{1}&\textbf{3}&\textbf{5}&\textbf{10}&\textbf{100}\\
        \midrule
        Laplace/LP \cite{zhu2003semi}&   27.9 (10.4)    &   47.6 (8.1)   &   56.0 (5.9)     & 63.7 (3.5)  &  71.6 (1.2)       \\
        Sparse LP \cite{jungsparselp2016} & 19.3 (2.4) & 21.9 (2.3) & 24.5 (2.3) & 28.9 (2.0) & 45.6 (2.2) \\
        p-laplace \cite{flores2022}&   39.7 (10.6)    & 51.4 (6.7)     &  56.1 (5.8)      & 60.9 (3.4)   &  69.5 (1.1)       \\
        p-eikonal \cite{calder2022hamilton}&    39.5 (10.6)   &   49.7 (7.4)   &   54.9 (6.3)     & 60.0 (3.5)   & 71.0 (1.0) \\       
        Poisson \cite{calder20poisson}      &   46.7 (12.1)   &  57.8 (7.4)      &     62.6 (4.4)    &  66.7 (2.5) & \textbf{74.1 (0.8)}     \\
        $p$-conductance ($p=1$, $\epsilon=n$)       & 24.8 (2.9)   &  24.7 (2.6)  & 25.9 (2.9)  &  28.7 (3.4)  & 40.3 (9.3)  \\
        $p$-conductance ($p=2$, $\epsilon=n$)       & 49.9 (11.4)  &  60.9 (4.5)  & 64.0 (3.0)  &  \textbf{67.2 (1.8)}  & 72.5 (0.8)  \\
        $p$-conductance ($p=\infty$, $\epsilon=n$)       & 37.9 (10.2)  & 47.6 (7.0)  & 53.4 (4.8)  &  58.5 (3.4)  & 69.6 (1.0)  \\[0.5ex]
        \cdashline{1-6}\noalign{\vskip 0.5ex}
        PoissonMBO \cite{calder20poisson}      &  46.2 (10.6)     &   55.8 (5.4)    &    60.0 (3.0)   & 63.5 (2.5)   & 73.2 (1.0)   \\
        \textbf{$p$-conductance ($p=2$, $\epsilon=0$)}       & \textbf{50.7 (11.2)}  &  \textbf{61.0 (4.2)}  & \textbf{64.0 (2.2)}  &  66.7 (1.7)  & \textbf{74.1 (0.8)}  \\
        \bottomrule
        \end{tabular}
        \end{sc}
        \end{small}
        \end{center}
    \end{table*}

\subsection{Partially labeled learning}

    In the partially labeled learning (PLL) setting, the labels associated with a fixed proportion of the training set are augmented with a set of chosen candidate classes. In this experiment on the CIFAR-100~\cite{cifar} dataset, the candidate classes are selected from among the 5 fine-grained class labels that belong to the superclass of the true class label. Thus in the PLL task, for each labeled image in the training data, we consider a \emph{subset} of fine-grained labels associated with each training image as being given. The task is to learn the fine-grained class assignments for all test images, leveraging the labeled images with candidate sub-classes. In Table \ref{tab:pll} we demonstrate the decrease in accuracy as we increase the size of the candidate class set.
    $p$-conductance-MBO out-performs Poisson-MBO by 2.4\%-4\%.
    
    \begin{table*}[htb!]
        \caption{CIFAR-100 in the PLL setting. Average accuracy over 100 trials with standard deviation in brackets. Label rate is 10 labels per class. }\vspace*{-.25cm}
        \label{tab:pll}
        \begin{center}
            \begin{small}
                \begin{sc}
                    \begin{tabular}{llllll}
                        \toprule
                        CIFAR-100 \# labels &\textbf{1}&\textbf{2}&\textbf{3}&\textbf{4}&\textbf{5}\\
                        \midrule
                        Label propagation (Zhu et al., '03)   &   39.3 (0.5) & 33.4 (0.8) & 26.7 (1.0) & 18.6 (1.4) & 11.4 (0.3)      \\
                        Poisson (Calder et al., '20)    &   40.5 (0.5) & 34.7 (0.6) & 28.0 (1.1) & 19.3 (1.3) & 11.3 (0.2)    \\ 
                        $p$-conductance ($p=1$, $\epsilon=n$)       &  21.4 (1.1)  &  20.1 (1.3) &   19.4 (1.1)  & 14.6 (1.9) & 11.4 (0.3) \\
                        $p$-conductance ($p=2$, $\epsilon=n$)       &  41.7 (0.4) & 35.0 (0.7) & 29.4 (1.1) & 22.2 (1.7) & 11.8 (0.1)  \\
                        $p$-conductance ($p=\infty$, $\epsilon=n$)    &  30.4 (0.8)  & 27.1 (0.7)  &  26.4 (0.9)  & 23.1 (0.9) & \textbf{14.3 (0.1)} \\  
                        [0.5ex]
                        \cdashline{1-6}\noalign{\vskip 0.5ex}
                        PoissonMBO       &  42.8 (0.9)  & 37.8 (0.7)  &  31.4 (0.7)   & 22.3 (1.3) &  11.4 (0.1) \\
                        $p$-conductance ($p=2$, $\epsilon=n$)       &  42.1 (0.5) & 38.2 (0.5) & 32.5 (0.9) & 24.1 (1.8) & 12.2 (0.3)  \\
                        $p$-conductance-MBO       &  \textbf{45.2 (0.4)}      &  \textbf{40.5 (0.4)} &  \textbf{34.4 (0.6)}   & \textbf{26.4 (1.1)} & 14.1 (0.1) \\
                        \bottomrule
                    \end{tabular}
                \end{sc}
            \end{small}
        \end{center}
    \end{table*}

\subsection{Convergence of Newton-ALM}

    In figure \ref{fig:convergence}, we explore the empirical convergence of Newton-ALM. Since Newton-ALM exploits second-order information of the problem. We demonstrate rapid convergence compared to first-order methods. Our comparison is conducted with respect to the ADMM-based method derived in  section~\ref{sec:admm}. Note that the per-iteration computational cost is comparable. The dominant step in the ADMM algorithm is the least squares update to the primal variables, necessitating the inversion of the graph Laplacian of the data graph (see remark \ref{rem:admm}). Likewise, the dominant cost of the ADMM step is also the inversion of a rank-1 perturbed SPSD Hessian with the same sparsity structure as the underlying graph. 

    We evaluate convergence with respect to the decay of the relative KKT residual with respect to iteration for $p=5$ on a $20\times 20$ lattice dataset, where the relative residual is $\eta_1 + \eta_2$ and convergence is characterized by the following criterion:
    \begin{equation}
        \begin{aligned}
            &\max\{\eta_1, \eta_2\}\leq \epsilon,\\
            & \eta_1 = \frac{||B^\top \phi - u|| + ||\phi^\top(\mu - \nu)-1||}{1 + ||u|| + ||\mu - \nu||},\:\: \eta_2 = \frac{||B^\top z + y(\mu - \nu)|| + ||u - \text{prox}_{s/\sigma_1}(u + \frac{1}{\sigma_1}z)||}{1 + ||\mu - \nu|| + ||u||}
        \end{aligned}
    \end{equation}
    and $\epsilon > 0$ is a given tolerance. In our experiments, we set $\epsilon = 10^{-4}$.

    \begin{figure}[!htb]
        \centering
        \includegraphics[width=0.4\linewidth]{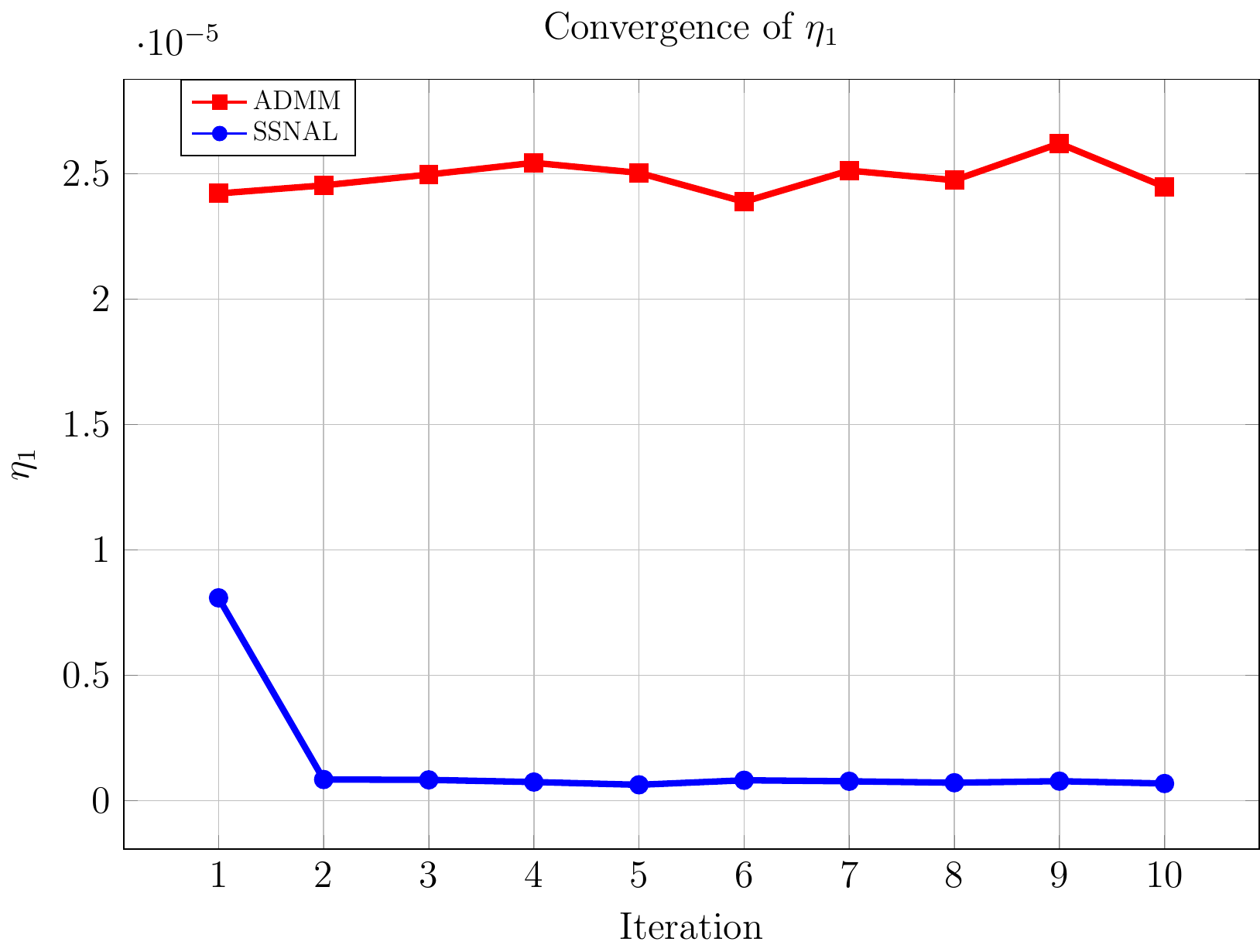}
        \includegraphics[width=0.4\linewidth]{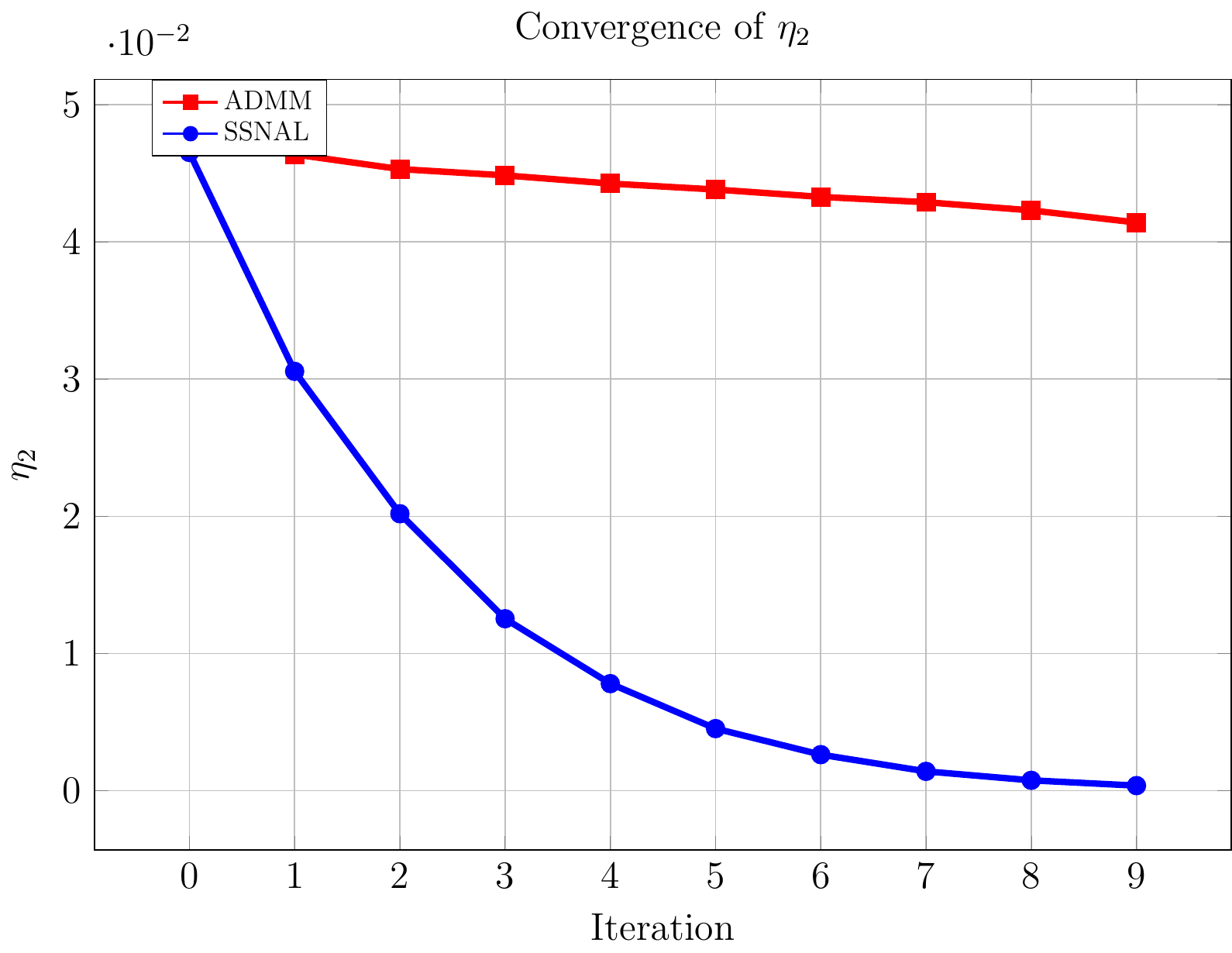}
        \caption{Convergence of KKT residuals on lattice data for $p=5$}
        \label{fig:convergence}
    \end{figure}
    
\subsection{Exploring the cuts and cut-boundaries}

    For simplicity, we first introduce a bipartitioning framework, $V_0(\phi_B) := \{v_i : (\phi_B)_i = 0\}$, $V_1(\phi_B) := \{v_j : (\phi_B)_j = 1\}$ characterized by a binary vector $\phi_B \in \{0, 1\}^M$. Equivalently, each binary vector $\phi_B$ determines an edge set $E_{\phi_B} = \{(i, j) : (\phi_B)_i = 0, (\phi_B)_j = 1\}$ connecting two subsets $V_0(\phi_B)$, $V_1(\phi_B)$. 
    
    In Figure \ref{fig:mnistboundary} and \ref{fig:fmnistboundary}, we provide examples of samples that lie on the ``boundary'' of their respective
    partitions. The node boundary of a set $V_1$ with respect to a set $V_2$ is the set of vertices $v$ in $V_1$ such that for some node $u$ in $V_2$, there is an edge joining $u$ to $v$. We hypothesize that challenging samples lie on the cut-boundary between partitions.

    In \cref{fig:illustration-cuts-cora}, we illustrate the solution $\phi$ to \cref{eq:bcut-intro} when $p=2$ on the Cora dataset, using measures obtained from randomly sampled neural network papers and the rest of the classes, respectively.
    
    \begin{figure}[htb!]
        \centering
        \begin{subfigure}[b]{0.49\linewidth}
            \includegraphics[width=\linewidth]{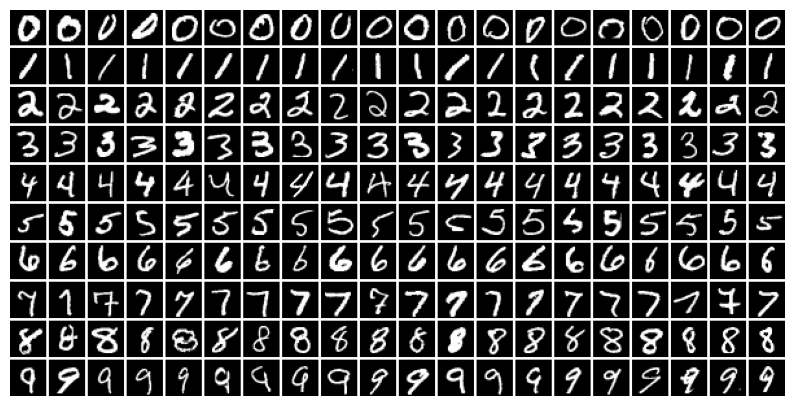}
            \subcaption{Random images}
            \label{fig:randommnistimages}
        \end{subfigure}
        \begin{subfigure}[b]{0.49\linewidth}
            \includegraphics[width=\linewidth]{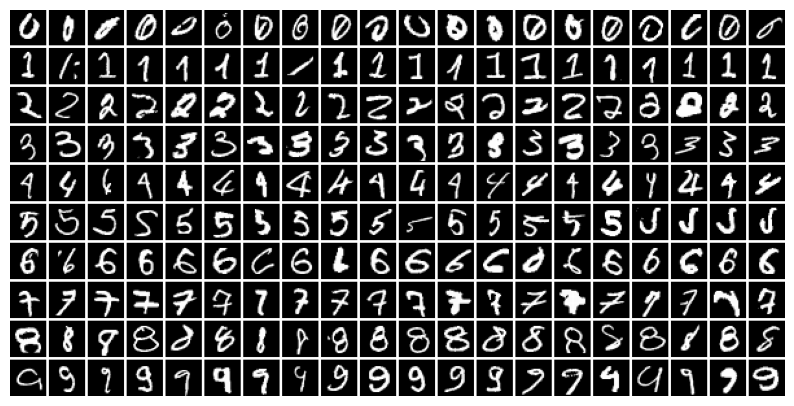}
            \subcaption{Boundary images}
            \label{fig:boundarymnistimages}
        \end{subfigure}
        \caption{MNIST experiments}
        \label{fig:mnistboundary}
    \end{figure}

    \begin{figure}[ht!]
        \centering
        \begin{subfigure}[b]{0.49\linewidth}
            \includegraphics[width=\linewidth]{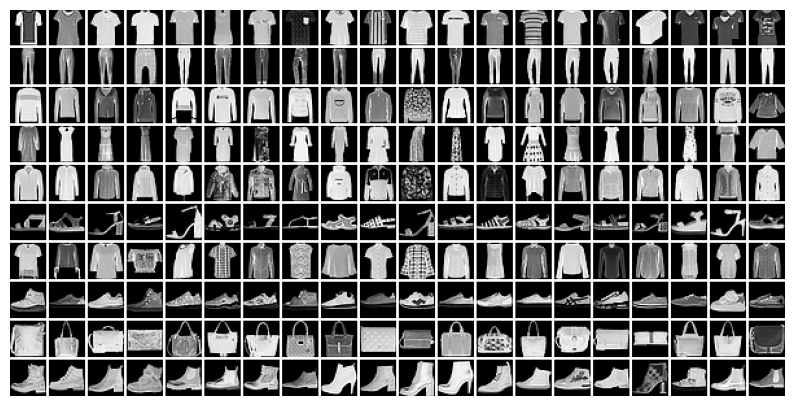}
            \subcaption{Random images}
            \label{fig:randomfmnistimages}
        \end{subfigure}
        \begin{subfigure}[b]{0.49\linewidth}
            \includegraphics[width=\linewidth]{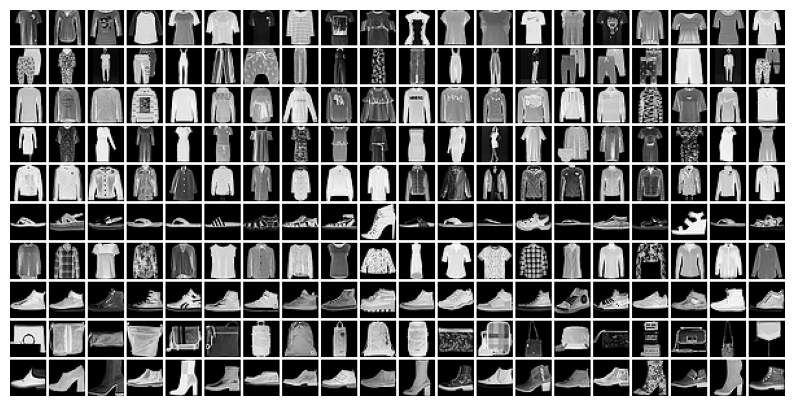}
            \subcaption{Boundary images}
            \label{fig:boundarymnistimages}
        \end{subfigure}
        \caption{FashionMNIST experiments. The labels are 0: T-shirt/top 1: Trouser 2: Pullover 3: Dress 4: Coat 5: Sandal 6: Shirt 7: Sneaker 8: Bag 9: Ankle boot}
        \label{fig:fmnistboundary}
    \end{figure}

    \begin{figure}[htb!]
        \begin{center}
            \includegraphics[width=0.5\textwidth]{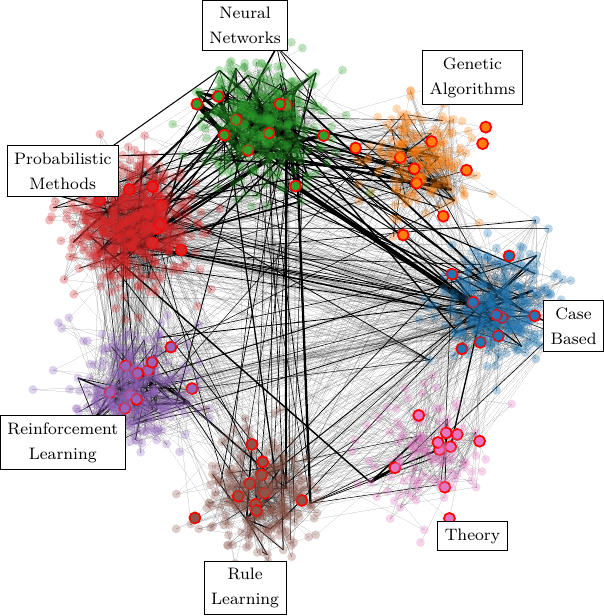}
        \end{center}\caption{An illustration of the one-vs-all minimum measure cut on a graph obtained from the Cora dataset~\cite{planetoid}. Here, the measure $\mu$ is obtained from five randomly selected nodes among all neural network papers and $\nu$ is obtained from five randomly selected nodes from each of the other classes. The width of each edge $\{i, j\}$ is proportional to $|\phi_i-\phi_j|$ when $|\phi_i-\phi_j|$ is large and edges are suppressed otherwise, where $\phi$ solves \cref{eq:bcut-intro} for $p=2$. The training nodes used for constructing $\mu,\nu$ are distinguished by size and red outlines.}\label{fig:illustration-cuts-cora}
    \end{figure}

    \newpage

\section{Proofs from \Cref{sec:properties-and-theory}}\label{sec:proofs-sec-2}

    \begin{proof}[Proof of \cref{th:grl-max-flow-min-cut}]
        The proof here comes in two parts: \textbf{(1)}, showing that \cref{eq:linear-program-c1} and \cref{eq:bcut-intro} for $p=1$ are the same problem, and \textbf{(2)}, deriving \cref{eq:dual-linear-program-c1} as the Lagrangian dual to \cref{eq:linear-program-c1}. The final observation follows from strong duality for linear programs, since both the primal and dual problems are feasible because $G$ is connected. We remind the reader that $n = |V|$ and $N = |E|$.
        
        \textbf{(1)} The only difference between the linear program \cref{eq:linear-program-c1} and \cref{eq:bcut-intro} for $p=1$ is the introduction of a dummy variable $k_{ij}$. To eliminate it, the constraint $k_{ij}\geq \psi_i - \psi_j$ suggests that to minimize the objective we set $k_{ij} = \psi_i - \psi_j$ when $\psi_i - \psi_j \geq 0$, and $k_{ij} = 0$ otherwise. Thus, the linear program has the reduced form
            \begin{align*}
                \inf \left\{\sum_{e=\{i, j\}\in E}w_{ij}|\psi_i - \psi_j| : (\mu-\nu)^T \psi\geq 1\right\},
            \end{align*}
        which is equal to \cref{eq:bcut-intro}.
        
        \textbf{(2)} Write the variables for the primal problem in the form:
            \begin{align}
                X = \begin{bmatrix}
                    k \\
                    \psi_+\\ 
                    \psi_-
                \end{bmatrix}\in\mathbb{R}^{2N + n + n},
            \end{align}
        where we introduce the dummy variables $\psi_-, \psi_+\in\mathbb{R}^n$ so that $\psi = \psi_+ - \psi_-$ with $\psi_\pm\geq 0$, and introduce the matrix
            \begin{align}
                C^T = \begin{bmatrix}
                    I_{2N\times 2N} & -\widetilde{B}^T & \widetilde{B}^T\\
                    0_{2N}^T & (\mu-\nu)^T & -(\mu-\nu)^T
                \end{bmatrix}\in\mathbb{R}^{(2N+1)\times (2N + n + n)}.
            \end{align}
        Lastly, let the fixed vector $c\in\mathbb{R}^{2N+n+n}$ be given by
            \begin{align}
                {c} = \begin{bmatrix}
                    w\\
                    0_{n}\\
                    0_{n}
                \end{bmatrix}\in\mathbb{R}^{2N + n + n},
            \end{align}
        where $w\in\mathbb{R}^{2N}$ is the vector of edge weights indexed by $E'$. Then \cref{eq:linear-program-c1} can be written
            \begin{align}\label{eq:condensed-lp}
                \begin{cases}
                    \text{minimize }& {c}^T{X}\\
                    \text{subject to }&X\geq 0\\
                    &C^T X\geq \begin{bmatrix}
                        0_{2N}\\
                        1
                    \end{bmatrix}
                \end{cases}
            \end{align}
        Then, from the general form of the linear program in \cref{eq:condensed-lp}, we can obtain the dual form
            \begin{align}\label{eq:condensed-dual-lp}
                \begin{cases}
                    \text{maximize }& \begin{bmatrix}
                        0_{2m}\\
                        1
                    \end{bmatrix}^T Y\\
                    \text{subject to }&Y\geq 0\\
                    &C Y\leq c
                \end{cases}.
            \end{align}
        Therefore, we write $Y^T = \begin{bmatrix}J^T &f \end{bmatrix}^T$ for $J\in\mathbb{R}^{2N}$ and $f\in\mathbb{R}$, to obtain, upon inspection, \cref{eq:dual-linear-program-c1}.
    \end{proof}

    \begin{proof}[Proof of \cref{thm:mincuts-random-cuts}]
        Letting $\phi\in\mathbb{R}^n$ be fixed, we note first that if $T\sim\mathrm{Unif}([0, \|\phi\|_\infty])$, then 
            \begin{align*}
                \ep{\# E(A_T, B_T)} &= \sum_{\{i,j\}\in E} w_{ij}\pr{\min_{x\in \{i, j\}}\phi(x)< T \leq \max_{y\in \{i, j\}}\phi(y)}\\
                &= \frac{1}{\|\phi\|_\infty}\sum_{\{i,j\}\in E} w_{ij}|\phi_i-\phi_j|
            \end{align*}
        where $A_T = \{\phi_i \geq T\}$ and $B_T = \{\phi_i < T\}$. Next we have that, assuming $\phi \geq 0$, it holds
            \begin{align*}
                \ep{\sum_{i\in A_T}\mu_i-\nu_i} &= \sum_{i\in V}(\mu_i-\nu_i)\pr{\phi_i > T}\\
                &= \sum_{i\in V}(\mu_i-\nu_i)\frac{\phi_i}{\|\phi\|_\infty} = \frac{1}{\|\phi\|_\infty}
            \end{align*}
        Note that the optimal value of $\cp$ as defined by the program \cref{eq:bcut-intro} is unchanged when the added assumption $\phi\geq 0$ is given.
    \end{proof}

\subsection{Gauge duality background and proofs}\label{subsec:gauge}

    In this subsection we cover some background from gauge duality and provide a proof of \cref{thm:generalized-gauge-duality}. As a linearly constrained (weighted) norm minimization problem, $\cp$ occurs as an instance of a broader class of so-called \emph{gauge optimization problems}. These problems have been studied for their special properties since at least~\cite{freund1987dual}. In particular, following the convention of \cite{friedlander2014gauge}, we consider optimization problems of the form
        \begin{align}\label{eq:primal-gauge}
            \begin{cases}
                \text{minimize } &\kappa(x)\\
                \text{subject to }&\rho(Ax - b) \leq \sigma\\
            \end{cases}
        \end{align}
    where $x\in\mathbb{R}^n$, $A\in\mathbb{R}^{m\times n}$, $b\in\mathbb{R}^m$, $\sigma\in\mathbb{R}$, and $\kappa:\mathbb{R}^n\rightarrow\mathbb{R}$, $\rho:\mathbb{R}^m\rightarrow\mathbb{R}$ are gauge functions (i.e., nonnegative, positively homogeneous, and convex functions which vanish at the origin). For such problems, there exists a dual program known as the gauge dual (distinguished from the Lagrangian dual), which can be written in the form
        \begin{align}\label{eq:gauge-dual}
            \begin{cases}
                \text{minimize } &\kappa^0(A^T y)\\
                \text{subject to }& b^T y - \sigma \rho^0(y) \geq 1 \\
            \end{cases}
        \end{align}
    where the gauge polar function $\kappa^0$ (resp. $\rho^0$) is the function which best satisfies the Cauchy-Schwartz-type inequality $ y^T x \leq \kappa(x)\kappa^0(y)$, or more precisely,
        \begin{align*}
            \kappa^0(y) = \inf \left\{\mu >0 : y^T x \leq \mu \kappa(x), \text{ for each }x\in\mathbb{R}^n\right\},
        \end{align*}
    and which is also a gauge function (not that the gauge polar to a norm is the dual norm). As was shown in \cite{friedlander2014gauge}, there exist sufficient conditions under which the primal problem \cref{eq:primal-gauge} and the gauge dual \cref{eq:gauge-dual} display strong duality, which in this setting means that $v_p v_d = 1$ where $v_p$ and $v_d$, respectively, are the optimal values of the primal and dual gauge programming problems.    

    \begin{proof}[Proof of Proposition~\ref{thm:generalized-gauge-duality}]
        Assume $1 < p < \infty$ and $1/p+1/q=1$. We work from the Beckmann metric, noting that \cref{eq:defn-beckmann-1} has the form
            \begin{align*}
                \mathcal{B}_{w^{1-q}, q}(\mu,\nu) \;=\;\text{minimize }\|J\|_{w^{1-q}, q}
                \quad
                \text{subject to }
                \|BJ - (\mu-\nu)\|_2\leq 0.
            \end{align*}
        With $\kappa(\cdot) = \|\cdot\|_{w^{1-q}, q}$, and $\rho(\cdot) = \|\cdot\|_2$ we have that the gauge polar function is the dual norm, i.e., $\kappa^0(\cdot) = \|\cdot\|_{w, p}$ via \cref{lemma:norm-duality} and $\rho^0 = \rho$. Thus we have that the gauge dual program is
            \begin{align*}
                \text{minimize }\|B^T\phi\|_{w, p}
                \quad
                \text{subject to }
                \phi^T(\mu-\nu) \geq 1.
            \end{align*}
        which of course is $\mathcal{C}_p$. Strong gauge duality holds for feasible linearly constrained norm optimization problems (see, e.g., by Corollary 5.2 in \cite{friedlander2014gauge}), and thus the claim follows. The cases of $p=1,\infty$ are similar.
    \end{proof}

    \begin{lemma}\label{lemma:norm-duality}
        Let $w_1,w_2,\dotsc, w_n>0$ be fixed positive weights, and for $x\in\mathbb{R}^n$ and $1\leq p\leq \infty$ define
            \begin{align*}
                \|x\|_{w,p} &= \begin{cases}
                    \left(\sum_{i=1}^n w_i |x_i|^p\right)^{1/p} &\text{ if }1\leq p <\infty\\
                    \max_{1\leq i \leq n}w_i |x_i| &\text{ if }p=\infty\\
                \end{cases}.
            \end{align*}
        Then the norm dual to $\|\cdot\|_{w,p}$, denoted $\|\cdot\|_{w,p}^\ast$, is
            \begin{align*}
                \|y\|_{w,p}^\ast &= \begin{cases}
                    \max_{1\leq i \leq n}\frac{|y_i|}{w_i} &\text{ if }p=1\\
                    \|y\|_{w^{1-q},q} &\text{ if }1<p<\infty\text{ and }1/p+1/q=1\\
                    \|y\|_{w^{-1},1} &\text{ if }p=\infty\\
                \end{cases}
            \end{align*}
        for $y\in\mathbb{R}^n$.
    \end{lemma}

    \begin{proof}
        We begin with the case of $p=1$. We must compute
            \begin{align*}
                \|y\|_{w,1}^* &= \sup_{\|x\|_{w,1}\le1} x^T y = \sup_{x\in\mathbb{R}^n} \left\{ x^T y : \sum_{i=1}^n w_i|x_i|\le1 \right\}.
            \end{align*}
        If $\sum_{i=1}^n w_i|x_i|\le1$, then 
            \begin{align*}
                \sum_{i=1}^n x_i\,y_i 
                \le \sum_{i=1}^n |x_i|\,|y_i|
                \le \bigl(\max_i \tfrac{|y_i|}{w_i}\bigr)\sum_{i=1}^n w_i\,|x_i|
                \le \max_{1\le i\le n}\,\frac{|y_i|}{w_i}.
            \end{align*}
        Thus the supremum is at most $\max_i \tfrac{|y_i|}{w_i}$. To see that this value is achieved, fix an index $k$ such that $\tfrac{|y_k|}{w_k}=\max_i \tfrac{|y_i|}{w_i}$. Setting $x_k := \tfrac{\mathrm{sgn}(y_k)}{w_k}$ and $x_i=0$ for $i\neq k$, one verifies $x^Ty = \max_i \tfrac{|y_i|}{w_i}$. The case of $p=\infty$ then follows similarly since finite-dimensional normed spaces are reflexive.

        Now assume $1<p<\infty$ holds. Let $q$ satisfy $\frac1p+\frac1q=1$. By definition, $\|x\|_{w,p}\le1$ holds if and only if $\sum_{i=1}^n w_i\,|x_i|^p \le1$. Using Hölder's inequality, we write
            \begin{align*}
                x^T y
                =\sum_{i=1}^n x_i\,y_i
                =\sum_{i=1}^n \Bigl(x_i\,w_i^{\tfrac1p}\Bigr)\Bigl(y_i\,w_i^{-\tfrac1p}\Bigr)
                \le
                \Bigl(\sum_{i=1}^n w_i\,|x_i|^p\Bigr)^{\!\tfrac1p}
                \Bigl(\sum_{i=1}^n w_i^{-\tfrac qp}\,\lvert y_i\rvert^q\Bigr)^{\!\tfrac1q}
                \le
                \Bigl(\sum_{i=1}^n w_i^{-\tfrac qp}\,\lvert y_i\rvert^q\Bigr)^{\!\tfrac1q},
            \end{align*}
        since $\sum_{i=1}^n w_i\,|x_i|^p \le1$. Thus
        $\sup_{\|x\|_{w,p}\le1} \langle x,y\rangle \le
        \Bigl(\sum_{i=1}^n w_i^{-\tfrac qp}\,\lvert y_i\rvert^q\Bigr)^{\!\tfrac1q}$. To see that equality holds, one may choose $x$ proportional to 
        $\bigl(w_i^{-\tfrac1p}\,\mathrm{sgn}(y_i)\,|y_i|^{\tfrac{q-1}{q}}\bigr)$. In that case, the Hölder bound becomes an equality, and we conclude that
        $\|y\|_{w,p}^* = \bigl(\sum_{i=1}^n w_i^{-\tfrac qp}\,\lvert y_i\rvert^q\bigr)^{\tfrac1q}$. We note finally that $w_{i}^{-\frac{q}{p}} = w_{i}^{-q(1-\frac{1}{q})} =  w_{i}^{1-q}$.
    \end{proof}

\subsection{Proofs of stability results}\label{subsec:pf-stability}

    \begin{proof}[Proof of \cref{thm:robustness}]
        We note first that
            \begin{align*}
                \|\psi - \widetilde{\psi}_t\|_2 &\leq \| L^+ (I-e^{-tL})(\mu-\nu)\| + \| L^+ e^{-tL}\eta\|.
            \end{align*}
        Since $L^+$ and $e^{-tL}$ may be mutually diagonalized by the same orthogonal matrix $U$, we have
            \begin{align*}
                \| L^+ e^{-tL}\eta\| &= \| U \mathrm{diag}(0, \dotsc, 0, \lambda_k^{-1}e^{-t\lambda_k},\dotsc, \lambda_n^{-1}e^{-t\lambda_n}) U^T\eta\|\\
                &\leq \lambda^{-1}e^{-t\lambda}\|\eta\|
            \end{align*}
        where we assume $k\geq 2$ is the index first nonzero eigenvalue of $L$, and the eigenvalues are ordered in ascending manner. On the other hand, using the bound $e^{x}\geq 1+x$ which holds for all $x$ and the same diagonalization reasoning as before, we have
            \begin{align*}
                \|L^+(I-e^{-tL})\| &= \max_{i \geq k} \lambda_{i}^{-1} (1-e^{-t\lambda_i}) \leq \max_{i \geq k} \lambda_{i}^{-1} (1-(1-t\lambda_i)) = t,
            \end{align*}
        and therefore \cref{eq:improved-bound} follows. Thus, we look for $t$ which satisfy
            \begin{align*}
                t\|\mu-\nu\|_2 + \lambda^{-1}e^{-t\lambda}\|\eta\|_2 < \frac{1}{\lambda}\|\eta\|_2,
            \end{align*}
        which is equivalent to 
            \begin{align}\label{eq:equivalent}
                t\lambda\frac{\|\mu-\nu\|_2}{\|\eta\|_2} + e^{-t\lambda} < 1.
            \end{align}
        By \cref{lem:calculus}, assuming $\frac{\|\eta\|_2}{\|\mu-\nu\|_2} > 1$, if $t< \frac{1}{\lambda}\left(\frac{\|\eta\|_2}{\|\mu-\nu\|_2} - 1\right)$, we have that \cref{eq:equivalent} holds, from which the rest of the claim follows.
    \end{proof}

    \begin{lemma}\label{lem:calculus}
        Let $0 < \alpha < 1$ be fixed. If $ 0 < x < \frac{1}{\alpha} - 1$, then $\alpha x + e^{-x} < 1$.
    \end{lemma}

    \begin{proof}
        From $e^x \geq 1+x$ we have $\alpha x + e^{-x} \leq \alpha x + \frac{1}{x+1}$. Since $ x < 1/\alpha - 1$ we have $\alpha - \frac{1}{x+1} <0$ which implies $x\alpha -\frac{x}{1+x} <0$ since $x>0$. But then
            \begin{align*}
                x\alpha -\frac{x}{1+x} = x\alpha +\frac{1}{x+1} -1 < 0,
            \end{align*}
        from which the claim follows.
    \end{proof}

\section{Proofs from \Cref{sec:algorithms}}\label{sec:proofs-sec-3}
    
    \begin{proof}[Proof of \cref{prop:prox-semismooth}]
        Recall the definition of $s(u)$: $s(u) = \sum_{i \leq N}w_i|u_i|^p$ and the associated proximal operator:
        $$
        F(v) := \text{prox}_{\lambda s}(v) = \min_u \left \{\frac{1}{2}||u - v||_2^2 + \lambda s(u)\right \}
        $$
        Note that $s(u)$ is coordinate-seperable. Likewise, the proximal operator of a seperable function decomposes coordinatewise. Consider
        $$
        f(u) = \frac{1}{2}(u - x)^2 + \alpha  |u|^p,\:\: \alpha > 0.
        $$
        Consider the scalar proximal mapping
        $$
        \hat{u}(x) = \argmin_{u \in \mathbb{R}}f(u).
        $$
        We will show that $u \to \bar{u}(x)$ is piecewise polynomial (and hence Lipschitz and strongly semismooth).
    
        For $p=1$, when
        $$
        f(u) = \frac{1}{2}(u - x)^2 + \alpha  |u|
        $$
        the well‐known closed‐form solution is:
        $$
        \bar{u}(x) = \text{sign}(x)\max\{|x| - \alpha, 0\}
        $$
        This map is the usual``soft‐threshold'' operator, which is clearly piecewise linear in $x$ and thus strongly semismooth.
    
        For the integer cases where $p\geq 2$, the function $|u|^p$ is actually a \emph{smooth} polynomial in $u$ (including at $u=0$). Indeed, for even $p$, it is $u^p$, and for odd $p$, it is $|u|^p$ which is $u|u^{p-1}|$, but still in $\mathcal{C}^1$ if $p\geq 2$. Consequently, $f(u) = \frac{1}{2}(u-x)^2 + \alpha |u|^p$ is a smooth (differentiable) function on $\mathbb{R}$. The minimizer $\bar{u}(x)$ is given by solving $f'(u) = 0$, i.e.
        $$
        f'(u) = (u-x) + \alpha p |u|^{p-1}\text{sign}(u) =0
        $$
        For each fixed $x$, this is a smooth equation in $u$. Moreover, the implicit function $\bar{u}(x)$ obtained by solving $f'(u) = 0$ is a $\mathcal{C}^1$ function of $x$, and thus \emph{strongly} semismooth.
    \end{proof}

    \begin{proof}[Proof of \cref{thm:ssnal-convergence}]
    From \cref{prop:prox-semismooth}, we know that $\text{prox}_{\lambda s(u)}$ is strongly semismooth for any $\lambda > 0$. By \cite{zhao2010newtoncgal} [Proposition 3.3], we know that $\Delta \phi$ obtained in SSNCG is a descent direction, which guarantees that the algorithm is well-defined. From \cite{zhao2010newtoncgal} [Theorem 3.4, 3.5], the desired result follows.
    \end{proof}

\section{Additional algorithmic details}

\subsection{Cut assignment in the $k$-way case} \label{sec:cut-assignment}
    
    Recall that in the $k$-way case, $\phi \in \mathbb{R}^{n\times k}$. In this section we describe one heuristic to map the potentials $\phi$ to one of $k$ partitions $V_1, \ldots V_k$. Represent the assignment of node $i$ to partition $r$ as the $r$-th standard basis vector $e_r = (0,\ldots,1,0,\ldots,0)$. Denote the assignment matrix $P \in \mathbb{Z}^{n\times k}$, with the $i$-th row of $P$ corresponding to the assignment associated with the $i$-th node, i.e.,
    $
    P_{ij} = \begin{cases}
                1 & \text{if } i \in V_j \\
                0 & \text{if } i \not \in V_j
             \end{cases}.
    $ 
    Assuming an estimate of the cardinality of the supports of each partitioning is given by a vector $m \in \mathbb{Z}_{\geq 0}$ with $m_i = |V_i|$ and $\mathbf{1}_k^\top m = n$, we intend to solve the binary linear program
    \begin{equation}
        \max_{P \in \{0,1\}^{n\times k}} \left \{ \langle \phi, P \rangle : P\mathbf{1}_k = \mathbf{1}_n,\:\: P^\top \mathbf{1}_k = m \right \}
    \end{equation}
    The constraints ensure that $P$ is a feasible cut-matrix\textemdash that for each row (node) only one partition is assigned and that for each partition $r$, exactly $m_r$ vertices lie in that partition. 

    We solve the following linear relaxation, which allows violation of cardinality constraint via the parameter $\epsilon \geq 0$:
    \begin{equation}
        \begin{aligned}
            \max_{P \in \mathbb{R}^{n\times k}} \{ \langle \phi, P \rangle :& P \geq 0,\:\: P\mathbf{1}_k = \mathbf{1}_n,\\
            &m -\epsilon \leq P^\top \mathbf{1}_k \leq m +\epsilon \}
        \end{aligned}
    \end{equation}
    \begin{remark}
        When $\epsilon = 0$, this is a linear program over the \emph{transportation polytope}. It is known that the simplex method recovers integer solutions. See \cite{deloera2013combinatoricsgeometrytransportationpolytopes}. When $\epsilon = n$, the optimum corresponds to the standard heuristic of thresholding the largest corresponding component of the vector $\phi^\ast$, i.e.,
        \begin{equation}\label{eq:labeldec}
            \phi^\ast_i = \argmax_{j\in \{1,\dots,k\}} \{\phi^\ast_{ij}\}.
        \end{equation}
    \end{remark}       

\subsection{$p$-conductance-MBO method}

    Empirically, the effectiveness of exact cut-based methods on $k$-nn graphs is known~\cite{calder20poisson, holtz2024continuous}. In particular, methods combinatorially search the discrete space of labels to refine continuous-valued predictions have demonstrate strong accuracy. However, the price is nonconvexity and computational cost. In order to properly evaluate the effectiveness of our method, we augment our approach with a cut-based refinement strategy by building off the prior state of the art.
    
    \cite{calder20poisson} propose a graph-cut method to incrementally adjust a seed decision boundary so as to improve the label accuracy and account for prior knowledge of class sizes. The proposed method applies several steps of gradient descent on the graph-cut problem:
    \begin{equation}
        \min_u \frac{1}{2} ||\nabla u||^2_{\ell^2(X)^2} + \frac{1}{\tau}\sum_{i=1}^n\prod_{j=1}^k|u(x_i) - u_j|^2
    \end{equation}
    More concretely, the time-spitting scheme that alternates gradient descent on two energy functionals is employed:
    \begin{equation}\label{eq:poissonmboiterations}
    \begin{aligned}
        &E_1(u) := \frac{1}{2} ||\nabla u||^2_{\ell^2(X)^2} - \mu\sum_{j=1}^m (y_k - \bar{y})\cdot u(x_j) \\ 
        &E_2(u) := \frac{1}{\tau}\sum_{i=1}^n\prod_{j=1}^k|u(x_i) - u_j|^2
    \end{aligned}
    \end{equation}
    The first term $E_1$ corresponds to the Poisson learning objective. Gradient descent on the second term $E_2$, when $\tau > 0$ is small, amounts to projecting each $u(x_i) \in \mathbb{R}^k$ to the closest label vector $e_j \in S_k$, while preserving the volume constraint $(u)X = b$. This projection is approximated by the following procedure: Let $Proj(S_k) : \mathbb{R}^k \to S_k$ be the closest point projection, let $s_1, \ldots , s_k > 0$ be positive weights, and set
    \begin{equation}
        u^{t+1}(x_i) = \text{Proj}_{S_k}(\text{diag}(s_1, \ldots , s_k)u^t(x_i))    
    \end{equation}
    where $\text{diag}(s_1, \ldots , s_k)$ is the diagonal matrix with diagonal entries $s_1, \ldots , s_k$. A simple gradient descent scheme to choose the weights $s_1, \ldots , s_k > 0$ is employed so that the volume constraint $(u^{t+1})X = b$ holds.

    We propose an analogous method that solves the proximity-regularized cut problem
    \begin{equation}
        \begin{aligned}
            &\min_{P \in \mathbb{R}^{n\times k}} \alpha \langle P^\top, LP \rangle - \langle \phi, P \rangle \\
            &\text{s.t. } P \in \{0,1\}^{n\times k}, \:\: P1_k = 1_n, P^\top 1_n = m\:\:
        \end{aligned}
    \end{equation}
    where $\alpha > 0$ is a weighting parameter balancing alignment with the $p$-conductance potentials vs. cut penalty. To solve this problem, we propose a splitting method which alternates between taking gradient steps to minimize the objective with projections onto the constraints. We term this method $p$-conductance-MBO due to it's similarity with the classic MBO method. Denote $\mathcal{P}$ the convex hull of the constraint set, i.e.,$\mathcal{P} = \{P : P \geq 0, \:\: P1_k = 1_n, P^\top 1_n = m\}$. The $p$-conductance-MBO procedure is summarized below.
        \begin{algorithm}[H]\caption{$p$-conductance-MBO}
            \begin{flushleft}
                \textbf{Input:} potentials $\phi$, step size $\eta$ 
            \end{flushleft}
            \begin{algorithmic}[1]\label{alg:beckmann-mbo}
                \STATE Initialize $P_0 \in \{0,1\}^{n\times k}$ 
                \WHILE{not converged} 
                \STATE Compute \\
                $P^{t+\frac{1}{2}} \gets P^t - \eta\left\{\nabla_{P^t}\alpha \langle (P^t)^\top, LP^t \rangle - \langle \phi, P^t \rangle\right\}$
                \STATE $P^{t + 1} = \arg\max_{P \in \mathcal{P}}\langle P, P^{t+\frac{1}{2}} \rangle$
                \ENDWHILE
            \end{algorithmic}
        \end{algorithm}
    
\subsection{First-order splitting methods}
    \label{sec:admm}
    In this section, we demonstrate a first-order Alternating Direction Method of Multipliers (ADMM) method for solving \eqref{eq:bcut-primal}. Introduce auxiliary variables $v_{ij} =\phi_i - \phi_j \in \mathbb{R}^k$ and multipliers $\lambda_{ij} \in \mathbb{R}^k$. ADMM works by finding a saddle point of the Lagrangian subject to primal constraints: 
    \begin{equation}
        \max_\lambda \min_{\phi, V} l (\phi, V, \lambda)
    \end{equation}
    Where the Lagrangian function $l(\phi, V, \lambda)$ is given by
    \begin{equation}
        \begin{aligned}
            l(\phi, V, \lambda) = \sum_{i,j = 1}^n w_{ij}\left (||v_{ij}||^p_p + \lambda_{ij} \cdot (v_{ij} - \phi_i +
             \phi_j) + \frac{\nu}{2} ||v_{ij} - \phi_i + \phi_j||_2^2\right )
        \end{aligned}
    \end{equation}
    For some positive parameter $\nu > 0$. From \emph{Boyd, Convex Optimization}, ADMM consists of repeating the following three steps:
    
    \begin{enumerate}
        \item Each $v_{i,j}$ is a minimizer of
        \begin{equation}
            ||v_{ij}||^p_p + \frac{\nu}{2}||v_{ij} - z_{ij}||_2^2,
        \end{equation}
        Where $z_{ij} = \phi_i - \phi_j - \frac{1}{\nu}\lambda_{ij}$. The calculation of $v$ is given by the proximal map for $||\cdot||^p_p$.
        \item Each $\phi_{ij}$ is a minimizer of the following least squares problem
        \begin{equation}
            \sum_{i,j = 1}^n \frac{w_{ij}}{2}|| - \phi_i + \phi_j + v_{ij} + \frac{1}{\nu}\lambda_{ij}||^2\quad \text{s.t. } \phi^\top r = 1
        \end{equation}
        The computation is given in Remark 1.1.
        \item The update to the multipliers is given by
        \begin{equation}\label{eq:lambda}
            \lambda_{ij} \to \lambda_{ij} + \nu (v_{ij} - \phi_i + \phi_j)
        \end{equation}
    \end{enumerate}
    
    \begin{remark}\label{rem:admm}
        Let $p\in \{1,2\}$. We may assume that all matrices $\lambda$, $z$ and $v$ are all skew-symmetric. Indeed, let the matrix $\lambda$ be initially skew-symmetric. Then $z$ is skew-symmetric and $v$ is skew-symmetric as well from the definition the proximal map. And $\lambda$ is again skew-symmetric from \eqref{eq:lambda}. Hence, $\phi$ is a solution to the system
        \begin{equation}
            L \phi = W \odot (v + \frac{1}{\nu}\lambda)\mathbf{1}_n, \quad \phi^\top r = 1.
        \end{equation}
        
    \noindent Let $d_i = \sum_j w_{ij}$. Consider the first order condition for the partial derivatives: 
        \begin{equation}
            \sum_{j}w_{ij}(\phi_j - \phi_i + v_{ij} + \frac{1}{\nu} \lambda_{ij}) = 0
        \end{equation}
        Re-arranging, we get the system
        \begin{equation}
            \phi_i (\sum_j w_{ij}) = \sum_j w_{ij}(\phi_j + v_{ij} + \frac{1}{\nu}\lambda_{ij})
        \end{equation}
        Thus, we have the system
        \begin{equation}
            D \phi = W \phi + W \odot (v + \frac{1}{\nu} \lambda)\mathbf{1}_n
        \end{equation}
        Simplifying the above expression, we get 
        \begin{equation}
            L \phi =  W \odot (v + \frac{1}{\nu}\lambda)\mathbf{1}_n.
        \end{equation}
        Let $b=W \odot (v + \frac{1}{\nu}\lambda)\mathbf{1}_n$. Consider the constraint $\phi^\top r = 1$. The projection onto the constraint is calculated by $\phi = L^{\dagger}(b - \gamma r)$, where $\gamma = \frac{b^\top L^\dagger r - 1}{r^\top L^\dagger r}$.
    \end{remark}

\end{document}